\documentclass{article}

\usepackage{PRIMEarxiv}
\usepackage[utf8]{inputenc} 
\usepackage[T1]{fontenc}    
\usepackage{hyperref}       
\usepackage{nicefrac}       
\usepackage{microtype}      
\usepackage{fancyhdr}       
\usepackage[caption=false,font=normalsize,labelfont=sf,textfont=sf]{subfig}
\usepackage{textcomp}
\usepackage{stfloats}
\usepackage{url}
\usepackage{verbatim}
\usepackage{cite}
\usepackage{graphicx}
\usepackage{enumerate}
\usepackage{amsmath}
\usepackage{amssymb}
\usepackage{amsfonts,amsthm}
\usepackage{mathtools}
\DeclarePairedDelimiter\ceil{\lceil}{\rceil}
\DeclarePairedDelimiter\floor{\lfloor}{\rfloor}
\usepackage{booktabs}
\newcommand{\indep}{\perp \!\!\! \perp}
\newtheorem{theorem}{Theorem}
\newtheorem{corollary}{Corollary}
\newtheorem{lemma}{Lemma}[section]
\newtheorem{proposition}{Proposition}

\theoremstyle{remark}

\usepackage{algorithm}
\usepackage{algorithmic}

\title{Golden Ratio-Based Sufficient Dimension Reduction
}

\author{
  Wenjing Yang \\
  School of Statistics \\
  University of Minnesota - Twin Cities \\
  Minneapolis, Minnesota, US \\
  \texttt{yang2987@umn.edu} \\
   \And
  Yuhong Yang \\
    School of Statistics \\
  University of Minnesota - Twin Cities  \\
  Minneapolis, Minnesota, US \\
  \texttt{yangx374@umn.edu} \\
}

\begin{document}
\maketitle

\begin{abstract}
Many machine learning applications deal with high dimensional data. To make computations feasible and learning more efficient, it is often desirable to reduce the dimensionality of the input variables by finding linear combinations of the predictors that can retain as much original information as possible in the relationship between the response and the original predictors. We propose a neural network based sufficient dimension reduction method that not only identifies the structural dimension effectively, but also estimates the central space well. It takes advantages of approximation capabilities of neural networks for functions in Barron classes and leads to reduced computation cost compared to other dimension reduction methods in the literature. Additionally, the framework can be extended to fit practical dimension reduction, making the methodology more applicable in practical settings.
\end{abstract}


\section{Introduction}
\label{ch3.1}
The curse of dimensionality poses significant challenges to statistical analysis when dealing with a large number of variables \cite{bellman1959}. Under the supervised learning framework, sufficient dimension reduction (SDR) has emerged as a useful tool that bridges the gap between high dimensionality and traditional statistical modeling. However, current state-of-the-art statistical methods for SDR in the literature often presume the structural dimension, which is generally not the case in practice. Additionally, these methods may not be computationally feasible for handling large sample sizes or high dimensionality efficiently, which limits their usage in many real-world applications. 

Numerous methods have been proposed on SDR in the past decades, see e.g., \cite{li2018sufficient}. For classical methods such as sliced inverse regression (SIR) \cite{li1991}, minimum average variance estimation (MAVE) method \cite{xia2002adaptive}, and sliced average variance estimation (SAVE) \cite{cook1991}, a main class of estimators for the central space is based on the inverse conditional moments of $X | Y$, where $Y \in \mathbb{R}$ and $X \in \mathbb{R}^p$ are the response and $p$-dimensional predictor in a regression analysis, respectively. Slicing the continuous response $y$ is often used to facilitate the estimation. However, this imposes some strong probabilistic structure on $X$. Moreover, selecting the number of slices remains an open and challenging question \cite{wang2020}. \cite{zhu2010dimension} proposed a cumulative slicing estimation methodology for SDR and developed three methods: cumulative mean estimation (CUME), cumulative directional regression (CUDR), and cumulative variance estimation (CUVE). \cite{cook2014fused} introduced a fused estimation procedure (fSIR), with observed performance improvement in some situations. \cite{cai2020online} implemented a sliced inverse regression in an online fashion that first constructs
an online estimate for the kernel matrix and then performs online singular value decomposition. \cite{lin2018consistency} established consistency of estimation of the dimension reduction space in a high-dimensional setting. 
 
In a more recent work, \cite{wang2020} proposed an aggregate inverse mean estimation (AIME) procedure that may substantially improve estimation accuracy compared to the previous methods. It incorporates the cumulative slicing scheme into the aggregate SDR idea proposed by \cite{wang2020aggregate} and is much less sensitive to linearity condition violations with the localization step before aggregation. \cite{artemiou2021} proposed a real-time approach for SDR that uses a principal least squares support vector machines approach to estimate the central subspace more accurately. Their method updates the estimator efficiently as new data is collected, starting from an initial estimator obtained with currently available data. \cite{soale2021} proposed a method that first estimates the expectiles through kernel expectile regression and then carries out dimension reduction based on random projections of the regression expectiles. Several methods in the literature are extended under these frameworks \cite{luo2022efficient, wang2022structured, pircalabelu2022high}. There are also neural network approaches to SDR for tackling classification problems \cite{zhang2019learning, teng2023level, meng2020sufficient}. For regression problems, \cite{liang2022nonlinear} proposed a nonlinear SDR method and \cite{kapla2022fusing} proposed a stochastic neural network that is computationally feasible to handle large scale data. \cite{kapla2022fusing} has proposed an algorithm that is able to obtain structural dimension, although no theoretical understanding is provided (different from this work). 

In this paper, we propose a golden ratio-based neural network for SDR (GRNN-SDR), a novel approach that utilizes neural networks to capture complex functional forms that are previously inaccessible with the traditional nonparametric regression tools. Our algorithm incorporates the golden ratio to dynamically search for the structural dimension, which significantly reduces computation time and complexity. Theoretical basis have demonstrated the generalization ability of multi-layer neural networks \cite{barron1994approximation, klusowski2018approximation}. Under proper conditions, we establish theoretical results that demonstrate that our approach leads to the true structural dimension with high probability. Compared to most of the existing methods, which typically presume the structural dimension to estimate the central space, extensive numerical results show that our proposed method is able to obtain the true or practical structural dimension effectively without prior knowledge. Extensive experiment comparisons show that our method estimate the central space with higher accuracy in most cases and demonstrate higher stability when the true dimensionality is not small. Furthermore, our algorithmic complexity under a fixed neural network structure is $O(N)$, where $N$ is the sample size, offering a promising solution to the challenges in SDR.

\section{Method}\label{ch3.3}
\subsection{Preliminary}\label{ch3.3.1}
Let $Y \in \mathbb{R}$ and $X \in \mathbb{R}^p$ be the response and a $p$-dimensional
predictor in a regression analysis, respectively. The goal of SDR
is to replace $X$ with a small set of linear combinations,
$\beta^{T}X$ ($\beta$ consists of $d$ columns, where $d \leq p$), without loss of its regression information \cite{cook1994using}. The response $Y$ depends on the predictor $X$ through $\beta^{T}X$, meaning that we want to find $\beta \in \mathbb{R}^{p \times d}$ with $d \leq p$ such that 
\begin{equation} 
Y \indep X|\beta^{T}X, 
\label{eq0}
\end{equation}
where $\indep$ represents statistical independence. If there exists a full column ranked $\beta$ satisfying (\ref{eq0}), and for all $w \in \mathbb{R}^{p \times d}$, if $rank(w) < d$ indicates $Y \indep X|w^{T}X$ does not hold, then the column space of $\beta$ is the central space for the regression $Y$ on $X$, denoted as $\boldsymbol{S}_{Y|X}$ and the dimension of the central space is called the structural dimension, denoted as $d_{Y|X} = d$ \cite{cook1994using}. If the distributions of $Y | X$ and $Y | \beta^{T}X$ are the same, then dimension reduction can be achieved. This suggests that for $Y = f(X) + \epsilon$, where $X$ and $\epsilon$ are independent with var($\epsilon$) = $\sigma^2$ and
E($\epsilon$) = 0, we intend to find $g(z)$ and $\beta$, s.t. $f(x) = g(\beta^{T}x)$. 

\subsection{Approximation Bound}\label{ch3.3.2}
Let $g_m(z, \theta) = \sum^m_{i=1}\tau_i\phi(u^T_iz+v_i) + \tau_0$ be a neural network of one hidden layer that is parameterized by $\theta =\{u_{i}, v_{i}, \tau_{i},\tau_0 \mid i=1,2,...m\}$ as the set of parameters, with $u_i \in \mathbb{R}^d$, $v_i$, $\tau_i$, and $\tau_0 \in \mathbb{R}$, for $i = 1,\cdots,m$, where $m \geq 1$ is the number of nodes in the hidden layer, and active function $\phi(\cdot)$ is a sigmoidal function that has bounded derivative on $\mathbb{R}$ and satisfies the Lipschitz condition, $|\phi(z) - \phi(z^*)| \leq a_1|z - z^*|$ for some $a_1 > 0$, and $|\phi(z)| \leq a_0$ where $a_0 \geq 1$. We first give some definitions below based on \cite{barron1994approximation}. Define the accuracy of an approximation $g_{m}(z, \theta)$ to the target function $g(z)$ in terms of $L_2$-norm for a probability measure $\mu$ on a closed and bounded set in $\mathbb{R}^d$ and contains the point $x=0$ as 
\begin{equation} 
\|g - g_{m}\| = \sqrt{\int|g(z) - g_{m}(z, \theta)|^2\mu(dz)}.
\end{equation}

\noindent Define the first absolute moment of the Fourier magnitude distribution for functions $g(z)$ on $\mathbb{R}^d$ with a Fourier representation of the form $g(z) = \int_{\mathbb{R}^d}e^{i2\pi \eta^Tz}\tilde{g}(\eta)d\eta$ as
\begin{align*} 
C_g = \int\|\eta\|_1|\tilde{g}(\eta)|d\eta,
\end{align*} 
where $\|\eta\|_1 = \sum^d_{j=1}|\eta_j|$ is the $l_1$-norm of $\eta \in \mathbb{R}^d$, and $\tilde{g}(\eta) = \int_{\mathbb{R}^d} e^{-i2\pi \eta^T z} g(z)dz$. The set of functions with $C_g$ bounded is called a Barron class, which provides an efficient representation via neural networks. 

Let $f_{k,m}(x,\theta,w)$ be a neural network of two hidden layers defined as 
\begin{equation}
    f_{k,m}(x,\theta,w) = \sum_{i=1}^{m}\tau_{i}\phi(u_{i}^{T}I(w^{T}x)+v_{i}) + \tau_0, \\
    \label{eq:2}
\end{equation}
where the number of nodes, $k$, in the first hidden layer ($k=p$ initially) is equal to the number of column vectors of $w= (w^{(1)}, \cdots, w^{(k)}) \in \mathbb{R}^{p \times k}$, the weight vector from the input layer to the first layer, and $m$ is now the number of nodes in the second hidden layer. Define 
$(\theta, w) \in \Theta_{d,m, \delta, C} \overset{\Delta}{=} \{u_{i}, v_{i}, \tau_{i}, \tau_{0}, w \mid \sum^m_{i=1}|\tau_{k}|<C, |\tau_{0} - g(0)|<C, |v_{i}| \leq \|u_{i}\|_1, \|u_{i}\|_1 \leq \delta, \|w^{(j)}\|_1 \leq \delta\}$, $u_{i}$ and $v_{i}$ are the parameters from the first hidden layer to the second hidden layer, $\tau_{i}$ and $\tau_0$ are the parameters from the second hidden layer to the output layer, constant $C \geq C_g$, $\delta$ is a sufficiently large constant such that $\sup_z|\phi(\delta z) - \text{sign}(z)| \leq \frac{1}{\sqrt{m}}$. 

\begin{proposition}\label{prop3.1}
If $f(x) = g(\beta^{T}_dx)$, and $C_g$ is the finite first absolute moment of the Fourier magnitude distribution of $g$. Then there exists a neural network model of two hidden layers,\\
\begin{equation} 
f_{d,m}(x,\bar{\theta},\beta_d) = \sum_{i=1}^{m}\bar{\tau}_{i}\phi(\bar{u}_{i}^{T}I(\beta^{T}_dx)+\bar{v}_{i}) + \bar{\tau}_0, \\
 \label{eq1}
\end{equation}
such that
\begin{equation} 
\big\|f(\cdot)-f_{d,m}(\cdot,\bar{\theta},\beta_d)\big\| \leq \frac{2C_g}{\sqrt{m}}, 
\end{equation}
where $(\bar{\theta}, \beta_d) \in \Theta_{d,m, \delta, C}$. 
\end{proposition}

\noindent \textit{Remark.}
Using the identity function as the activation function and excluding bias from the input layer to the first layer of $f_{d,m}(x,\bar{\theta},\beta_d)$ together with the approximation result of \cite{barron1993universal} are crucial for obtaining the structural dimension and central space. This allows us to leverage the advantage of approximation results for neural networks.

Approximation of $g(z)$ can be obtained by a model selection method with multiple layers if needed. Adding the current two layer network to deep layers does not affect dimension approximation. Specifically, let $g_0 \overset{\Delta}{=} z$, $g_i(g_{i-1}) = \phi((w^{(i-1)})^Tg_{i-1} + b^{(i-1)})$ for $i= 1, 2,  \cdots, L$. Define $G(z,\theta) = w^{(L)} * g_{L} \circ g_{L-1} \circ \cdots \circ g_1(z) + b^{(L)}$ as a neural network of $L \geq 1$ hidden layers, where $\theta = (w^{(0)},b^{(0)},\cdots,w^{(L)},b^{(L)})$ ( $(w^{(i-1)}, b^{(i-1)})$ are the parameters from the $(i-1)^{th}$ layer to the $i^{th}$ layer and $(w^{(L)},b^{(L)})$ are the parameters from the $L^{th}$ layer to the output layer). For  $f(x) = G(\beta^{T}_dx)$ and $\big\|G(\cdot) - G(\cdot,\bar{\theta}) \big\| \leq B$, if we add another layer $\beta^T_dx$ between the input layer and the first hidden layer with identity function $I(\cdot)$ as the activation function, then, for the neural network with $(L+1)$ hidden layers, $f_{d}(x,\bar{\theta},\beta_d) = G(I(\beta_d^Tx),\bar{\theta})$, we have the following,
\begin{equation}
    \big\|f(\cdot) - f_{d}(\cdot,\bar{\theta},\beta_d) \big\| \leq B,
\end{equation}

\noindent where $B$ can be different explicit bounds that may be derived under specific conditions in literature.

Extensive work has also been introduced in more recent times with the focus on multilayer neural networks using many hidden layers \cite{schmidhuber2015deep, kohler2016nonparametric, bauer2019deep, stinchcombe1999neural}. \cite{mhaskar2016deep} compared deep network and shallow network from an approximation theory approach. 
In the Sobolev space, the approximation bound can be further improved with two hidden layers \cite{de2021approximation}. For $d \in \mathbb{N}$, let $\zeta \in \mathbb{N}_0^d$ of $d$-tuple of nonnegative integers be a multi-index, where $\mathbb{N}_0^d = \{(\zeta_1, \cdots, \zeta_d)|\zeta_i \text{ is a non-negative integer}, i = 1, \cdots, d\}$. Denote $|\zeta| = \sum_{i=1}^d \zeta_i$, $\zeta! = \prod_{i=1}^d \zeta_i!$, and $x^{\zeta} = \prod_{i=1}^d x_i^{\zeta_i}$ for $x \in \mathbb{R}^d$ as the corresponding multinomial. For a multi-index $\zeta$, the multinomial coefficient is defined as follows,
\begin{align}
    \binom{|\zeta|}{\zeta} = \frac{|\zeta|!}{\zeta!}.
\end{align}
For $\Omega \subseteq \mathbb{R}^d$ and a function $f: \Omega \rightarrow \mathbb{R}$, denote the classical or distributional derivative of $f$ as follows, 
\begin{align}
    D^{\zeta}f = \frac{\partial^{\zeta}f}{\partial x_1^{\zeta_1} \cdots \partial x_d^{\zeta_d}},
\end{align}
and for $n^{th}$ order partial derivative, denote $P_{n,d} = \{\zeta \in \mathbb{N}_0^d: |\zeta| = n\}$.

\noindent For $d \in \mathbb{N}$, $1 \leq q \leq \infty$, define the Sobolev space,
\begin{align}
    W^{k,q}(\Omega) = \{f \in L^q(\Omega): D^{\zeta}f \in L^q(\Omega) \text{ for all } \zeta \in \mathbb{N}_0^d \text{ with } |\zeta| \leq k\},
\end{align}
where $k \in \mathbb{N}_0$, $\Omega \subseteq \mathbb{R}^d$ is open, and $L^q(\Omega)$ is the Lebesgue space.

\noindent For $q < \infty$, define the seminorms on $W^{k,q}(\Omega)$,
\begin{align}
    |f|_{W^{m,q}(\Omega)} = \left(\sum_{|\zeta| = m} \|D^{\zeta}f\|^q_{L^q(\Omega)}\right)^{1/q},
\end{align}
for $m = 0, \cdots, k$, and based on this, define the following norm, 
\begin{align}
    \|f\|_{W^{k,q}(\Omega)} = \left(\sum_{m=0}^k |f|^q_{W^{m,q}(\Omega)}\right)^{1/q}.
\end{align}

\noindent When $q = \infty$, define the seminorms on $W^{k,\infty}(\Omega)$ as
\begin{align}
    |f|_{W^{k,\infty}(\Omega)} = \max_{|\zeta| = m}\|D^{\zeta}f\|^q_{L^{\infty}(\Omega)},
\end{align}
where $m = 0, \cdots, k$, and based on this, define the following norm, 
\begin{align}
    \|f\|_{W^{k,\infty}(\Omega)} = \max_{0 \leq m \leq k}|f|_{W^{m,\infty}(\Omega)}.
\end{align}
\noindent Based on the previous results, we readily have the following proposition. 
\begin{proposition}\label{prop2}
Let $d,s \in \mathbb{N}$, $\delta > 0$, and $f(x) = g(\beta_d^Tx)$, where  $\beta_d$ is full ranked matrix of dimension $d$, $x \in \mathcal{B}$, where $\mathcal{B} \subset \mathbb{R}^p$ is a bounded open set, $\beta_d^T\mathcal{B} \subset [-M,M]^d$, and $g \in W^{s,\infty}([-M,M]^d)$. There exist a constants $\mathcal{C}_{d,s,g} > 0$ and an integer $\mathcal{C}_0>0$, such that, for every $\mathcal{C} \in \mathbb{N}$ with $\mathcal{C} > \mathcal{C}_0$, there exists a tanh neural network, $f^\mathcal{C}(x) = g^\mathcal{C}(I(\beta_d^Tx))$, of three hidden layers. The width of the first layer is $d$, the width of the second layer is at most $3\ceil{\frac{s}{2}}|P_{s-1,d+1}| + d(\mathcal{C}-1)$, and the width of the third layer is at most $3\ceil{\frac{d+2}{2}}|P_{d+1,d+1}|\mathcal{C}^d$, such that
\begin{align}
   \|f - f^N\|_{L^{\infty}(\mathcal{B})} \leq (1+\delta)\frac{(2M)^s\mathcal{C}_{d,s,g}}{\mathcal{C}^s},
\end{align}
where $\mathcal{C}_{d,s,g} = \frac{1}{s!}(\frac{3d}{2})^s|g|_{W^{s,\infty}([-M,M]^d)}$.
\end{proposition}

\subsection{Error Bound for Training Neural Network}\label{error bound}
In practical neural network training, let $f_{k,m}(x,\hat{\theta}_{k,N},\hat{w}_{k,N})$ be the prediction model trained from a neural network defined in (\ref{eq:2}) based on a training data, $\{x_i, y_i\}^N_{i=1}$, of sample size $N$, where 
\begin{align}
    (\hat{\theta}_{k,N},\hat{w}_{k,N}) = \arg\min_{(\theta, w) \in \Theta_{N}}\left(\frac{1}{N} \sum_{i=1}^{N}
(y_i - f_{k,m}(x_i,\theta, w))^2 + \lambda \frac{L_{k,m}(\theta, w)}{N}\right),
\end{align}
where $\lambda$ is a positive constant, $\Theta_{N}$ is a discrete set of parameter vectors $(\theta, w)$, and $L_{k,m}(\theta, w)$ is the complexity of $(\theta, w) \in \Theta_{N}$ for given $m$ and $N$ that satisfies the following \cite{barron1994approximation},
\begin{enumerate}
    \item $L_{k,m}(\theta, w)$ is a positive number,
    \item $\sum_{(\theta, w) \in \Theta_{N}}e^{-L_{k,m}(\theta, w)} \leq 1$.
\end{enumerate}
The error bound is established as follows. First, we analyze the error bound when $k = d$. Denote $a = a_1(\delta d + \delta + 1) + a_0 + 1$. The parameter vector $(\theta, w)$ in the sigmoidal network in
\begin{equation}
   f_{d,m}(x,\theta,w_d) = \sum_{i=1}^{m}\tau_{i}\phi(u_{i}^{T}I(w^{T}_dx)+v_{i}) + \tau_0
    \label{eq.2}
\end{equation}
consists of weights $u_i, v_i,\tau_i$ for $i = 1, \cdots, m$, and $\tau_0$, where $u_i \in \mathbb{R}^d, v_i, \tau_i, \tau_0 \in \mathbb{R}$, and $w_d = (w^{(1)}, \cdots, w^{(d)}) \in \mathbb{R}^{p \times d}$, where $w^{(j)} \in \mathbb{R}^p$ for $j = 1, \cdots, d$. For some constant $\delta > 0$, a continuous parameter space $\Theta_{d,m,\delta, C}$ contained in $\mathbb{R}^{md+pd+2m+1}$ is taken to be the set of all such $(\theta, w_d)$ for which $\|u_i\|_1 \leq \delta$, $|v_i| \leq \delta$, and $\|w^{(j)}\|_1 \leq \delta$. 

We control the precision with which the coordinates of the parameter vectors are allowed to be represented. For each $\epsilon >0$ and $C\geq 1$, let $\Theta_{d,m,\epsilon,\delta, C}$ be a discrete set of parameter points in $\mathbb{R}^{md+pd+2m+1}$ that $\epsilon$-covers $\Theta_{d,m,\delta, C}$ in the sense that, for every $(\theta, w_d)$ in $\Theta_{d,m,\delta, C}$, there is a $(\theta^*, w^*_d)$ in $\Theta_{d,m,\epsilon,\delta, C}$ such that for $i = 1,\cdots, m$ and $j = 1,\cdots, d$,
\begin{align}
    \label{eq.5}
    \|u_i - u^*_i\|_1 \leq \epsilon \\
    \label{eq.6}
    \|w^{(j)} - w^{(j)*}\|_1 \leq \epsilon  \\
    \label{eq.7}
    |v_i - v^*_i| \leq \epsilon  \\
    \label{eq.8}
    \sum^m_{i=1}|\tau_i - \tau^*_i| \leq C\epsilon \\
    \label{eq.9}
    |\tau_0 - \tau^*_0| \leq C\epsilon.
\end{align}

\begin{theorem}\label{theorem3.1}
For $f(x) = g(w^Tx)$, where $w \in \mathbb{R}^{p \times d}$, let $\Theta_N = \Theta_{d,m,\epsilon, \delta, C}$, where $\Theta_{d,m,\epsilon, \delta, C}$ defined above, and let $\epsilon \sim \left(\frac{(m+p)d}{N}\right)^{1/2}$. Then the statistical risk of the prediction model $f_{d,m}(x,\hat{\theta}_{d,N},\hat{w}_{d,N})$ satisfies
\begin{align}
E\big\|f(\cdot) - f_{d,m}(\cdot,\hat{\theta}_{d,N},\hat{w}_{d,N})\big\|^2 \leq O\left(\frac{C_g^2}{m}\right) + O\left(\frac{(m+p)d}{N}\log\left(\frac{N}{(m+p)d}\right)\right). 
\end{align}
For $m \sim C_g(\frac{N}{d}/\log{\frac{N}{d}})^{\frac{1}{2}}$,
\begin{align}
E\big\|f(\cdot) - f_{d,m}(\cdot,\hat{\theta}_{d,N},\hat{w}_{d,N})\big\|^2 \leq O\left(\left(\frac{d}{N}\log\frac{N}{d}\right)^{\frac{1}{2}}\right). 
\end{align}
\end{theorem}

The proof of Theorem~\ref{theorem3.1} is based on Lemma 3.1 and Lemma 3.2, proved in Appendix~\ref{proof: thm1}.

We introduce the following assumptions to extend Theorem~\ref{theorem3.1} with the following corollary.

\noindent \textbf{Assumption 1:} $f(x) = g(\beta^Tx)$, where $\beta \in \mathbb{R}^{p \times d}$ with $rank(\beta) = d$, and $C_g$ is the finite first absolute moment of the Fourier magnitude distribution of $g$ \cite{barron1994approximation}. \\
\noindent \textbf{Assumption 2:} For $k<d$, there exists $c > 0$ such that for any $w \in \mathbb{R}^{p \times k}$ and $\tilde{g}(z)$, where $z \in \mathbb{R}^k$, $\big\|f(\cdot) - \tilde{g}(w^T\cdot) \big\| \geq c$.
 
\begin{corollary}\label{cor3.1}
If \textbf{Assumptions 1} and \textbf{2} hold, the statistical risk of the prediction model $f_{k,m}(x,\hat{\theta}_{k,N},\hat{w}_{k,N})$ satisfies the following. 
\begin{enumerate}
    \item If $k<d$, then
\begin{equation}
E\big\|f(\cdot) - f_{k,m}(\cdot,\hat{\theta}_{k,N},\hat{w}_{k,N})\big\|^2  \geq c^2,
\end{equation}
where $c$ is a positive constant.
\item If $k \geq d$, then
\begin{align}
E\big\|f(\cdot) - f_{k,m}(\cdot,\hat{\theta}_{k,N},\hat{w}_{k,N})\big\|^2  
= O\left(\frac{C_g^2}{m}\right) + O\left(\frac{(m+p)k}{N}\log(\frac{N}{(m+p)k})\right).
\end{align}
For $m \sim C_g\left(\frac{N}{k}/\log{\frac{N}{k}}\right)^{\frac{1}{2}}$, 
\begin{align}
    E\big\|f(\cdot) - f_{k,m}(\cdot,\hat{\theta}_{k,N},\hat{w}_{k,N})\big\|^2  &= O\left(\left(\frac{k}{N}\log\frac{N}{k}\right)^{1/2}\right) \nonumber \\
    &= O\left(\left(\frac{1}{N}\log N\right)^{1/2}\right). 
\end{align}
\end{enumerate}
\end{corollary}

\subsection{Structural Dimension Analysis}\label{structural dim}

Given $1 \leq k \leq p$, consider the analysis in Corollary~\ref{cor3.1}, we know that when $k \geq d$, the estimation risk is of order $O(\frac{1}{N}\log{N})^{\frac{1}{2}}$, and when $1 \leq k < d$, its risk is lower bounded away from 0. To determine the true dimension, consider the following criterion (CR) to select $k$: For $1 \leq k \leq p$, 
\begin{equation}\label{eq.CR}
    CR(k) \overset{\Delta}{=} MSE_{va}(k) + k*pen(N, n_{va}), 
\end{equation}
where $MSE_{va}(k)$ is the mean squared error (MSE) on the validation data, and $pen(N, n_{va})$ is a penalty that depends on the training sample size, $N$, and validation sample size, $n_{va}$ ($0.1N \leq n_{va} \leq 0.3N$ in general for most machine learning applications). Let $\hat{d} = \arg\min_{1 \leq k \leq p}CR(k)$ be the predicted structural dimension of $d$ based on (\ref{eq.CR}). If the assumptions in Corollary~\ref{cor3.1} hold, then we have the following theorem.

\begin{theorem}\label{theorem3.2}
For $pen(N, n_{va}) \overset{\Delta}{=} ((\frac{1}{N}\log{N})^{\frac{1}{2}}+\frac{1}{\sqrt{n_{va}}})a_{N,n_{va}}$ with $a_{N,n_{va}} \rightarrow \infty$, and \\ $pen(N, n_{va}) \rightarrow 0$, we have 
\begin{equation}
    P(\hat{d} = d) \rightarrow 1.
\end{equation}
\end{theorem}

\subsection{Algorithms}\label{ch3.3.3}

The relationship between the second hidden layer and the prediction error is apparent from Proposition~\ref{prop3.1} in Section~\ref{ch3.3.2}. The proposition shows that an increase in the number of nodes in the second hidden layer leads to a reduction in prediction error and improves over traditional nonparametric methods. Based on this, we can fix the number of nodes in the second layer and search for the optimal number of nodes in the first hidden layer through a dimension reduction procedure.

For a neural network structure (NN) defined in (\ref{eq:2}), train $t$ times with loss $\frac{1}{N}\sum^N_{i=1}(y_i - f_{k,m}(x_i, \theta, w))^2 + \lambda\|(\theta,w)\|_1$ to obtain the prediction model with Algorithm~\ref{alg3.1}. Denote $(\hat{\theta},\hat{w})$ as the weights of the $l^{th}$ NN, $\hat{f}^{(l)}$, where $l = \arg\min_{j = 1, \cdots, t} \{MSE^{(j)}_{va}\}$, with $MSE^{(j)}_{va}$ as the MSE of the $j^{th}$ neural network on validation data, and 
$MSE_{va}(k)$ as the MSE of $\hat{f}^{(l)}$ on the validation data. 

\begin{algorithm}
	\renewcommand{\algorithmicrequire}{\textbf{Input:}}
	\renewcommand{\algorithmicensure}{\textbf{Output:}}
	\caption{Process of Neural Network Learning (NNL)}
    \label{alg3.1}
	\begin{algorithmic}[1]   
		\REQUIRE Training data, validation data, NN, $k$ , $t$\\
		\ENSURE $MSE_{va}(k)$, $\hat{w}_k$, $\hat{f}_k(\cdot,\hat{\theta},\hat{w}_k)$
		\STATE Train NN $t$ times to get $t$ neural networks, $\hat{f}^{(j)}, j=1, \cdots, t$
		\STATE Calculate $l = \arg\min_{j = 1, \cdots, t}{\{MSE^{(j)}_{va}\}}$
		\STATE $\hat{f}_k(\cdot,\hat{\theta},\hat{w}_k) = \hat{f}^{(l)}$ 
		\STATE $MSE_{va}(k) = MSE^{(l)}_{va}$
		\RETURN $MSE_{va}(k)$, $\hat{w}_k$, $\hat{f}_k(\cdot,\hat{\theta},\hat{w}_k)$.
	\end{algorithmic}  
\end{algorithm} 

\begin{algorithm}
  \renewcommand{\algorithmicrequire}{\textbf{Input:}}
  \renewcommand{\algorithmicensure}{\textbf{Output:}}
  \caption{Prediction model and central space}
  \label{alg3.2}
  \begin{algorithmic}[1]
    \REQUIRE Training data, validation data, $p$ \\
    \ENSURE $\hat{d},\hat{\beta},  f_{\hat{d},m}(x,\hat{\theta},\hat{\beta})$
    \STATE $m_0 = 1, n_0 = p$ // Initialize $m_0$, $n_0$
    \STATE Compute $k_1$ by (\ref{eq.k1}), $k_2$ by (\ref{eq.k2})
    \STATE $S_1 = \text{MSE}_{va}(k_1)$, $S_2 = \text{MSE}_{va}(k_2)$, $S_{n_0} = \text{MSE}_{va}(n_0)$  // Call Algorithm~\ref{alg3.1} 
    \WHILE{$n_0 - m_0 \geq 4$}
    \STATE $n_0 = k_2, S_{n_0} = S_2$
    \IF{$S_2 - S_{n_0} \leq (n_0-k_2)*pen $ and $S_1-S_2 \leq (k_2-k_1)*pen$}
    \STATE $k_2 = k_1, S_2 = S_1$
    \STATE Update $k_1$ by (\ref{eq.k1})
    \STATE $S_1 = \text{MSE}_{va}(k_1)$ // Call Algorithm~\ref{alg3.1}
    \ELSE
    \STATE $m_0 = k_1$, $k_1 = k_2, S_1 = S_2$
    \STATE Update $k_2$ by (\ref{eq.k2})
    \STATE $S_2 = \text{MSE}_{va}(k_2)$ // Call Algorithm~\ref{alg3.1}
    \ENDIF
    \ENDWHILE
    \STATE $l=n_0$, $S(l) = MSE_{va}(l)$, $S(l-1) = MSE_{va}(l-1)$ // Call Algorithm~\ref{alg3.1}
    \WHILE{$S(l-1) - S(l) \leq pen$}
	\STATE $l = l-1$
	\STATE $S(l-1) = MSE_{va}(l-1)$  // Call Algorithm~\ref{alg3.1}
	\ENDWHILE
    \RETURN $\hat{d} = l + 1, \hat{\beta} = \hat{w}_d$, $f_{\hat{d},m}(x,\hat{\theta},\hat{\beta})$
  \end{algorithmic}
\end{algorithm}

Then, the search for optimal SDR on the model $f_{\hat{d},m}(x,\hat{\theta},\hat{\beta})$ is established. The initial search interval is $[m_0,n_0]$, where $m_0=1$, $n_0=p$, and $p$ is the number of variables selected under a neural network. The number of nodes in the second hidden layer is $m$. Let the number of nodes in the first hidden layer of the three networks be $k_1$, $k_2$, and $n_0$, where
\begin{align}\label{eq.k1}
k_1 &= m_0 + \floor*{0.382(n_0 - m_0)},\\ 
k_2 &= m_0 + \floor*{0.618(n_0 - m_0)}. \label{eq.k2}
\end{align}

Train these three neural networks using Algorithm~\ref{alg3.1} to obtain the corresponding prediction models. Based on the theoretical analysis in the following, we may obtain the central space $\boldsymbol{S}_{Y|X}(\hat{\beta})$ and structural dimension $\hat{d}$ from Algorithm~\ref{alg3.2}.

The SDR procedure is explained in Algorithm~\ref{alg3.2}. The prediction model is $f_{\hat{d},m}(x, \hat{\theta},$ $\hat{\beta})$, where $\hat{d}$ is the number of nodes of the first hidden layer, $\hat{\beta}$ is the weight matrix of the first hidden layer, and $\hat{\theta}$ contains the set of parameters specified in Corollary~\ref{cor3.1}. Note that $\hat{d}$ is the prediction of the minimum of $d$, and $\hat{\beta}$ is a prediction of $\beta$ under dimension reduction. 

\begin{theorem}\label{theorem3.3}
Through Algorithm~\ref{alg3.2}, under the same conditions as Theorem~\ref{theorem3.2}, the predicted structural dimension $\hat{d} = d$ with probability going to 1 as $N \rightarrow \infty$. 
\end{theorem}

When $n_0 - m_0 < 4$, Algorithm~\ref{alg3.2} executes Line 16. Then, Algorithm~\ref{alg3.1} is executed at most 3 times. For $n_0 - m_0 \geq 4$, suppose Algorithm~\ref{alg3.2} executes Algorithm~\ref{alg3.1} $s$ times. Then,
\begin{align}
    p(0.618)^s \leq 4  \\
    \Rightarrow s = \ceil{1.44(\log_{2}p - 2)}
\end{align}

Each NN in Algorithm~\ref{alg3.1} is executed $t$ times, so the total number of NN that Algorithm~\ref{alg3.2} trains is no more than $t(s+3) = t(\ceil{1.44(\log_{2}p - 2)} + 3)$. Furthermore, for a fixed initial neural network structure, the algorithmic complexity is $O(N)$ under gradient descent optimization.

\subsection{Practical Dimension Reduction}\label{ch3.3.4}
In real applications, the assumption on existence of $g(z)$ and $\beta$ that satisfies $f(x) = g(\beta^Tx)$ in Section~\ref{ch3.3.1} is stringent as the equality may be only an approximation, i.e., there exists a minimum integer $d$ and $\beta_d \in \mathbb{R}^{p \times d}$ such that $\big\|f(\cdot) - g(\beta_d^T\cdot)\big\| \leq \delta_N$, for some suitably small $\delta_N >0$. Due to this, we propose $\delta_N$-approximation SDR and define $d \overset{\Delta}{=} d(\delta_N)$ as the $\delta_N$-approximation structural dimension and the column space of $\beta_d$ as the $\delta_N$-approximation central space if the following two assumptions are satisfied.

Note that in high-dimensional statistical theory, for the purpose of capturing the real difficulties in estimation or model selection, the true regression function is often allowed to change with the sample size. \\
\noindent \textbf{Assumption 3:} There exists a function $g$ of dimension $d$, and $\beta_d \in \mathbb{R}^{p \times d}$ and a small positive constant $\delta_N$ such that
\begin{align}
    \big\|f(\cdot) - g(\beta_d^T \cdot)\big\| \leq \delta_N,
\end{align}

\noindent and there exists a constant $c_N > \delta_N$ such that for any function $\tilde{g}(z)$ of dimension $d-1$ and any $\alpha \in \mathbb{R}^{p \times (d-1)}$, we have
\begin{align}
    \big\|f(\cdot) - \tilde{g}(\alpha^T \cdot)\big\| \geq c_N.
\end{align}

This assumption is intuitive and appealing when $c_N \gg \delta_N$ and $\delta_N$ is small enough for good or near optimal performance after dimension reduction. We call the relaxed setup as practical dimension reduction. Theoretical results on the approximation bound for $\delta_N$-approximation can be established as follows.

\begin{proposition}\label{prop3.2}
Based on \textbf{Assumption 3}, when $g$ is in the Barron class with $C_g < \infty$, there exists a neural network model of two hidden layers, $f_{d,m}(x,\bar{\theta}, \beta_d)$ such that 
\begin{align}
    \big\|f(\cdot) - f_{d,m}(\cdot,\bar{\theta}, \beta_d)\big\| < \delta_N + \frac{2C_g}{\sqrt{m}},
\end{align} 
and for any neural network model of two hidden layers, $f_{k,m}(x,\theta,\alpha_k)$, where $k \leq d-1$, we have
\begin{align}
     \big\|f(\cdot) - f_{k,m}(\cdot,\theta,\alpha_k)\big\| \geq c_N.
\end{align}

\end{proposition}

Theoretical results on the error bound and structural dimension analysis are also established in the following.

\begin{theorem}\label{theorem3.4}
Suppose there exists $g(z)$ ($z \in \mathbb{R}^d$) and $w \in \mathbb{R}^{p \times d}$ such that $\big\|f(x) - g(w^Tx)\big\| < \delta_N$, and $C_g$ is the finite first absolute moment of the Fourier magnitude distribution of $g$. For a constant $C \geq C_g$, $\theta_N = \theta_{d,m,\epsilon,\delta,c}$, $\epsilon \sim \left(\frac{(m+p)d}{N}\right)^{1/2}$, $(\hat{\theta}_{d,N},\hat{w}_{d,N}) = \arg\min_{(\theta, w) \in \theta_N}(\frac{1}{N}\sum^{N}_{i=1}(y_i - f_{d,m}(x_i, \theta,w))^2 + \frac{L_{k,m}(\theta,w)}{N})$, the statistical risk of the prediction model, $f_{d,m}(x,\hat{\theta}_{d,N},\hat{w}_{d,N})$, satisfies 
\begin{align}
    E\big\|f(\cdot) - f_{d,m}(\cdot,\hat{\theta}_{d,N},\hat{w}_{d,N})\big\|^2 \leq \delta_N^2 + O\left(\frac{C_g^2}{m}\right) + O\left(\frac{(m+p)d}{N}\log\frac{N}{(m+p)d}\right).
\end{align}
\end{theorem}

\begin{corollary}\label{cor3.2}
If \textbf{Assumption 3} holds, then the statistical risk of the prediction model $f_{k,m}(x,\hat{\theta}_{k,N},\hat{w}_{k,N})$ satisfies the following.
\begin{enumerate}
    \item If $k<d$, then 
    \begin{align} \label{eq3.97}
        E\big\|f(\cdot) - f_{k,m}(\cdot,\hat{\theta}_{k,N},\hat{w}_{k,N})\big\|^2 \geq c_N^2.
    \end{align}
    \item If $k \geq d$, then 
    \begin{align} 
        E\big\|f(\cdot) - f_{k,m}(\cdot,\hat{\theta}_{k,N},\hat{w}_{k,N})\big\|^2 &\leq \delta_N^2 + O\left(\frac{C_g^2}{m}\right) +  O\left(\frac{(m+p)k}{N}\log\frac{N}{(m+p)k}\right).
        \end{align}
        For $m \sim C_g\left(N/ \log N \right)^{\frac{1}{2}}$, 
    \begin{align} \label{eq3.99}
    E\big\|f(\cdot) - f_{k,m}(\cdot,\hat{\theta}_{k,N},\hat{w}_{k,N})\big\|^2  &\leq 
    \delta_N^2 + O\left((\frac{1}{N}\log N)^{1/2}\right). 
     \end{align}
     
\end{enumerate}
\end{corollary}

For the following result, we assume there is a separation of $\delta_N^2$ and $c_N^2$ in the sense that $c_N^2 = \omega(\delta_N^2 + (\frac{1}{N}\log{N})^{\frac{1}{2}}+\frac{1}{\sqrt{n_{va}}})$, where for positive sequences $a_n$ and $b_n$, $a_n = \omega(b_n)$ means $a_n/b_n \rightarrow \infty$ as $n \rightarrow \infty$. 

\begin{theorem}\label{theorem3.5}
If \textbf{Assumption 3} holds, and $c_N \gg \delta_N$, 
$m \sim C_g\left(N/ \log N \right)^{\frac{1}{2}}$, 
then for $pen(N, n_{va})$ satisfying
\begin{align}\label{eq.pen}
  \delta_N^2 + \left((\frac{1}{N}\log{N})^{\frac{1}{2}}+\frac{1}{\sqrt{n_{va}}}\right) \ll pen(N, n_{va}) \ll c_N^2,
\end{align} 
we have 
\begin{equation}
    P(\hat{d} = d) \rightarrow 1,
\end{equation}
where $\hat{d} = \arg\min_{1\leq k \leq p}CRA(k)$ and $CRA(k) = MSE_{va}(k) + k * pen(N,n_{va})$.
\end{theorem}

\section{Simulations}\label{ch3.4}
In this section, we conduct a series of simulation studies to evaluate the finite sample performance of GRNN-SDR under multiple scenarios with a two-layer neural network of 20 nodes in the second hidden layer. We provide a well-rounded analysis from seven different aspects: noise level, sample size, computation time, data distribution, initial feature dimension size, and approximation analysis. Evaluation is based on the measurement of estimation accuracy, and the network is trained $t=3$ times for the experiments conducted in this work. We compare our method with several existing methods, including SAVE, CUME, fSIR, AIME, and NN-SDR for a  comprehensive analysis. The algorithm from \cite{liang2022nonlinear} and \cite{kapla2022fusing} are not directly provided, thus are omitted for comparisons.

\subsection{Evaluation Metrics}\label{ch3.4.1}
The estimation accuracy of a central space is evaluated based on the vector correlation coefficient, $r$ \cite{knyazev2002principal}. Suppose $\underline{\beta} \in \mathbb{R}^{p \times d}$ is an orthonormal basis of the central space (column space of $\beta$), and $\underline{\hat{\beta}} \in  \mathbb{R}^{p \times \hat{d}}$ is an orthonormal basis of the estimated central space (column space of $\hat{\beta}$), the vector correlation $r$ between
$\boldsymbol{S}_{Y|X}(\hat{\beta})$ and $\boldsymbol{S}_{Y|X}(\beta)$ is 
\begin{equation}
r = \sqrt{|\underline{\hat{\beta}}^T(\underline{\beta}{\underline{\beta}}^T)\underline{\hat{\beta}}|},
\end{equation}
where $|\cdot|$ is the determinant of the matrix. Note that $r \in [0,1]$, and if $\hat{d} = d$, then the the larger the value $r$ is, the closer $\boldsymbol{S}_{Y|X}(\hat{\beta})$ is to $\boldsymbol{S}_{Y|X}(\beta)$; if $\hat{d} < d$, then the larger the value $r$ is, the closer $\boldsymbol{S}_{Y|X}(\hat{\beta})$ is to a subspace of $\boldsymbol{S}_{Y|X}(\beta)$; if $\hat{d} > d$, then $r = 0$.

In addition, to assess prediction performance, we record the MSE and standard error, which are calculated based on 10 replications of each experiment.

\subsection{Simulation Study on Noise Level}\label{ch3.4.2}
In this experiment, we focus on analyzing and comparing the performances under various noise levels using simulated data of sample size 2000, with 1000 samples that is split into $80\%$ for training and $20\%$ for validation, and the remaining 1000 samples for testing. The data is generated from the following distribution.

\noindent  \textbf{Model 1}: Let $V = (V^{T}_{1}, V^{T}_{2}, \cdots, V^{T}_{p})^{T}$ be an orthogonal matrix where $V_1, V_2, \cdots, V_p$ are row vectors. Let $A = (1.01V^{T}_{1}, 1.01V^{T}_{2},  1.02V^{T}_{3}, 1.1V^{T}_{4}, 1.03V^{T}_{5})^{T}$, $Z = AX = (Z_1, Z_2, \cdots, Z_5)^{T} $, where $X \sim Uniform_{p}(-1,1)$ with $p = 20$.
\begin{align*} 
Y &= Z_3 - Z_1Z_5 + 0.5Z_2^{2} + (Z_3 + 0.5Z_4) / (1 + Z_1^{2})    
+ \exp(0.5(Z_3-Z_4)) \times \sin(Z_2 - Z_5 + 1.5Z_3) + c\epsilon,
\end{align*}
where $\epsilon$ is standard Gaussian noise independent of $X$, and $c$ is a constant.

We provide two sets of figures for analysis displayed in Fig.~\ref{figure3.1}. Fig. 1(a)-(d) illustrate the estimation accuracy $r$ over 10 replications. For each method, we provide the boxplot of the accuracy, together with mean $\pm$ one standard deviation on its right. Under the same setting, Fig. 1(e)-(h) shows how MSE changes based on the predictive model under each iterations with different number of nodes using the training, validation, and test data. The number of nodes is decreased based on the dynamic search until the MSE is the smallest. Specifically, it can be seen from the subfigures on the top row that GRNN-SDR achieves the best estimation accuracy ($r$) under all tested scenarios. Fig. 1(e)-(h) indicate that, despite the various noise levels, GRNN-SDR can efficiently reduce the number of nodes with only a few iterations. Furthermore, the MSE is the smallest at the true dimension, indicating that GRNN-SDR is indeed able to find the structural dimension successfully. 

\begin{figure*}[h]
  \small
  \centering
  \subfloat[]{
    \includegraphics[width=.22\textwidth,  trim=10 5 30 25,clip]{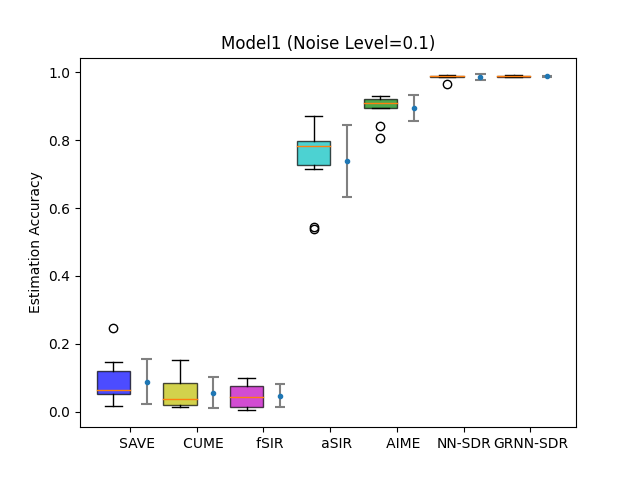}
    \label{fig:d}
  }
  \hfil
  \subfloat[]{
    \includegraphics[width=.22\textwidth, trim=10 5 30 25,clip]{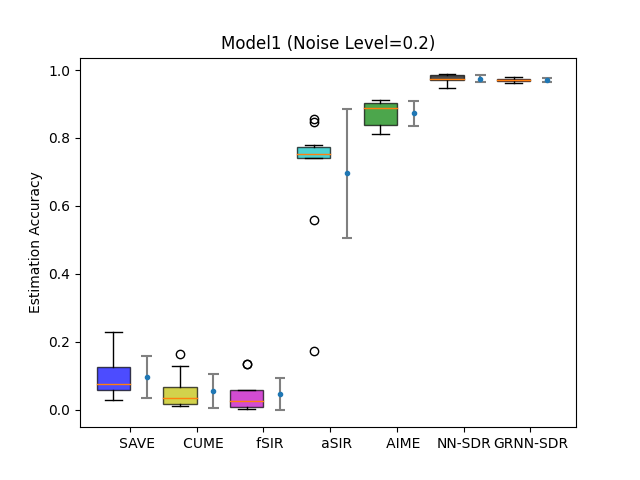}
    \label{fig:e}
  }
  \hfil
  \subfloat[]{
    \includegraphics[width=.22\textwidth, trim=10 5 30 25,clip]{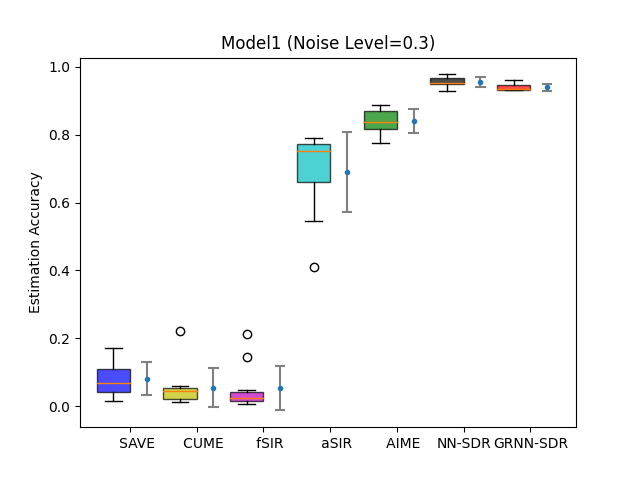}
    \label{fig:f}
  }
  \hfil
  \subfloat[]{
    \includegraphics[width=.22\textwidth,  trim=10 5 30 25,clip]{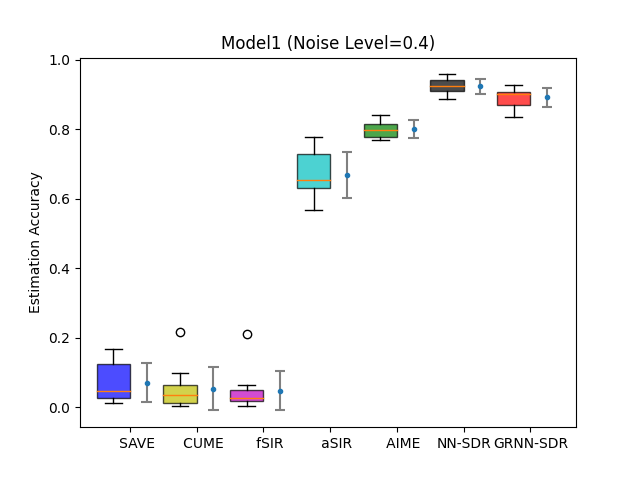}
    \label{fig:g}
  }
  \hfil
  \subfloat[]{
    \includegraphics[width=.22\textwidth,  trim=10 5 30 25,clip]{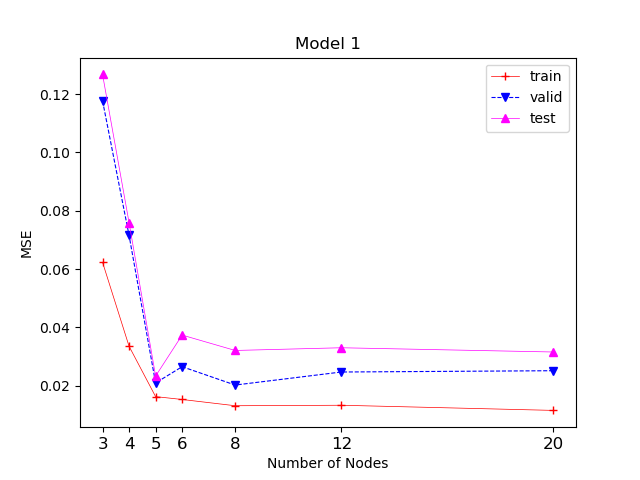}
    \label{fig:h}
  }
  \hfil
  \subfloat[]{
    \includegraphics[width=.22\textwidth,  trim=6 5 30 25,clip]{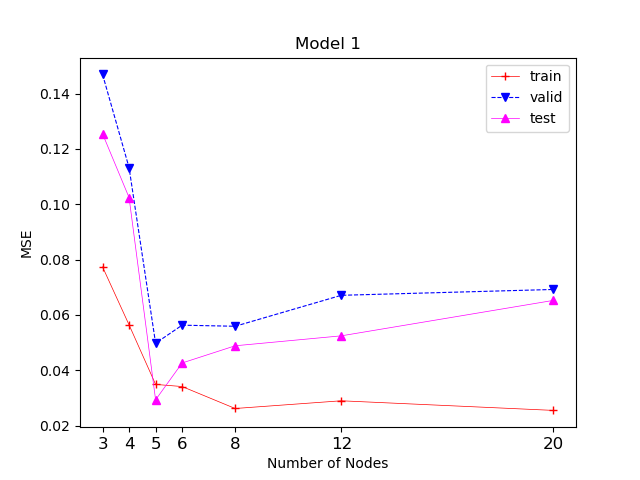}
    \label{fig:i}
  }
  \hfil
  \subfloat[]{
    \includegraphics[width=.22\textwidth,  trim=10 5 30 25,clip]{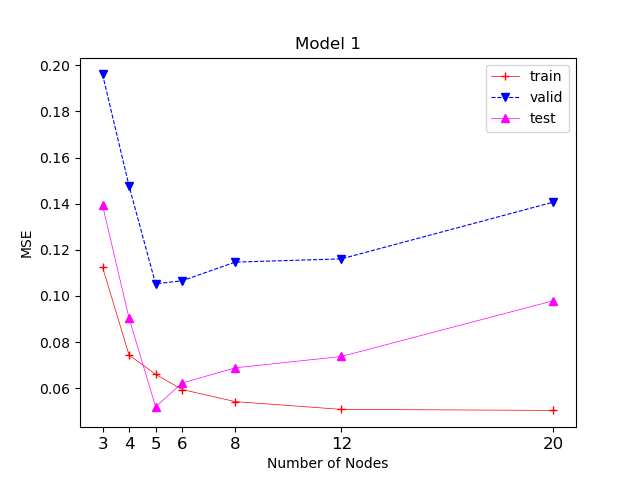}
    \label{fig:j}
  }
  \hfil
  \subfloat[]{
    \includegraphics[width=.22\textwidth,  trim=10 5 30 25,clip]{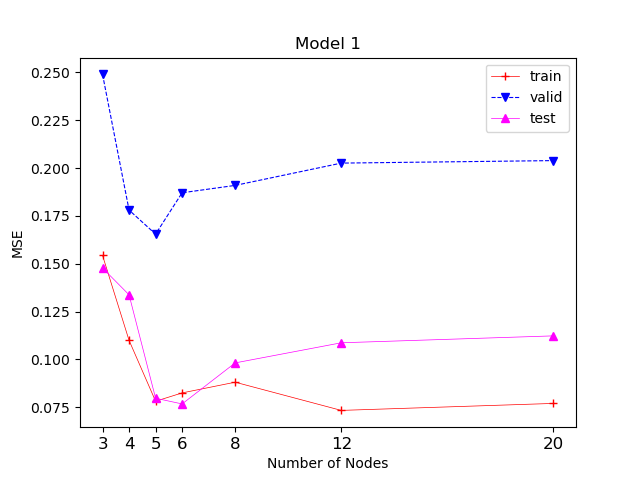}
    \label{fig:i}
  }
  \caption{Estimation accuracy ($r$) and MSE results for Model 1}
  \label{figure3.1}
\end{figure*}

\subsection{Simulation Study on Training Size and Computation Time}\label{ch3.4.3}
In this experiment, we evaluate the performance of different methods under various sample sizes for training and validation ($80\%$ for training and $20\%$ for validation). The test set sample size is fixed at 1000.

\noindent \textbf{Model 2}: Let $V = (V^{T}_{1}, V^{T}_{2}, \cdots, V^{T}_{p})^{T}$ be an orthogonal matrix where $V_1, V_2, \cdots, V_p$ are row vectors. Let $A = (V^{T}_{1}, V^{T}_{2}, V^{T}_{3}, V^{T}_{4}, V^{T}_{5})^{T}$, $Z = AX = (Z_1, Z_2, \cdots, Z_5)^{T} $, where $X \sim Uniform_{p}(-1,1)$, and $p = 20$.
\begin{align*} 
Y &= Z_5/(5+(1-0.2Z_3)^{2}) + \exp(0.5Z_1 + Z_2) + 2Z_1Z_4Z_5 \\
&+ (Z_4 - 0.5Z_1 + Z_2)\times \cos(0.5Z_3) + 0.2\sin(Z_1+Z_5) + 0.1\epsilon,
\end{align*}
where $\epsilon$ is standard Gaussian noise independent of $X$.

\begin{figure*}[!ht]
  \small
  \centering
  \subfloat[]{
    \includegraphics[width=.22\textwidth,  trim=5 5 35 25,clip]{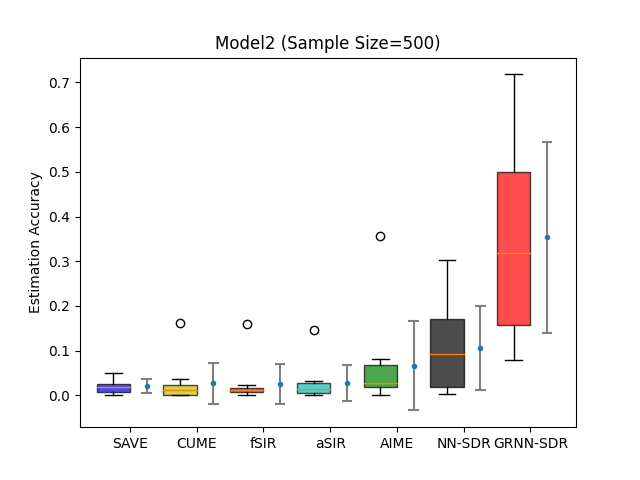}
    \label{fig:d}
  }
  \hfil
  \subfloat[]{
    \includegraphics[width=.22\textwidth, trim=10 5 35 25,clip]{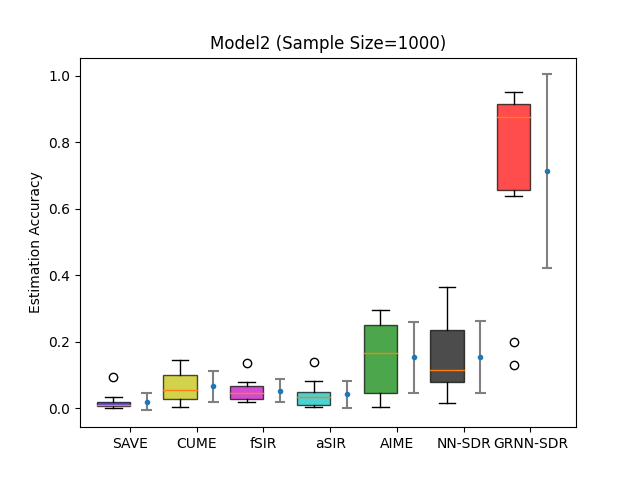}
    \label{fig:e}
  }
  \hfil
  \subfloat[]{
    \includegraphics[width=.22\textwidth, trim=10 5 35 25,clip]{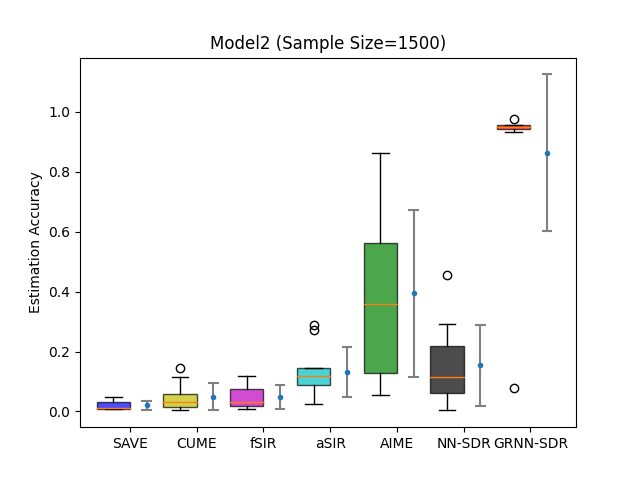}
    \label{fig:f}
  }
  \hfil
  \subfloat[]{
    \includegraphics[width=.22\textwidth,  trim=10 5 35 25,clip]{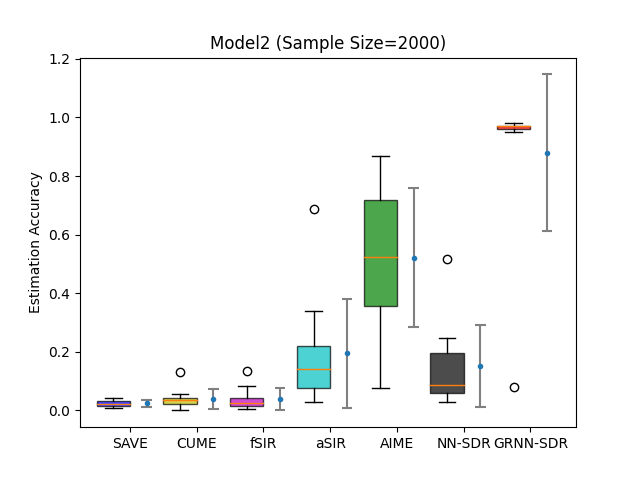}
    \label{fig:g}
  }
  \hfil
  \subfloat[]{
    \includegraphics[width=.22\textwidth,  trim=10 5 35 25,clip]{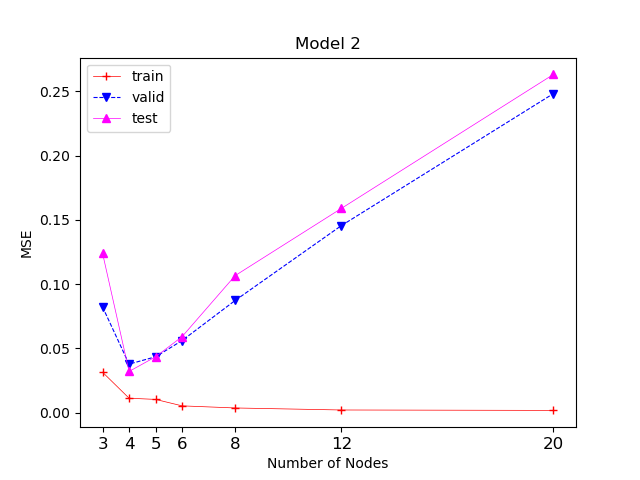}
    \label{fig:h}
  }
  \hfil
  \subfloat[]{
    \includegraphics[width=.22\textwidth,  trim=10 5 35 25,clip]{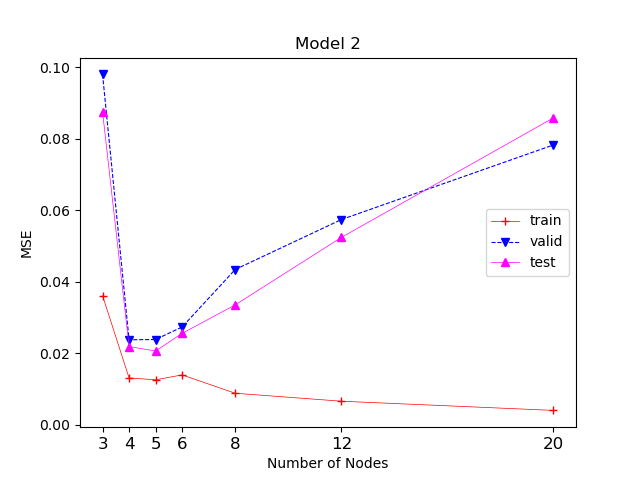}
    \label{fig:i}
  }
  \hfil
  \subfloat[]{
    \includegraphics[width=.22\textwidth,  trim=10 5 35 25,clip]{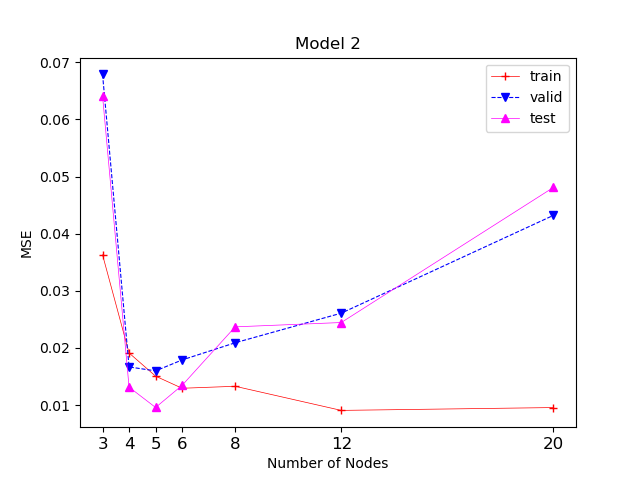}
    \label{fig:j}
  }
  \hfil
  \subfloat[]{
    \includegraphics[width=.22\textwidth,  trim=10 5 35 25,clip]{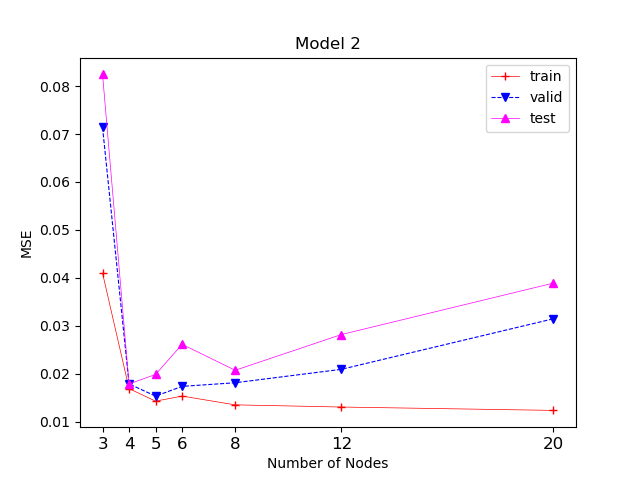}
    \label{fig:i}
  }
  \caption{Estimation accuracy ($r$) and MSE results for Model 2}
  \label{figure3.2}
\end{figure*}

Fig. 2(a)-(d)  shows that, with a small sample size, the estimation accuracy of GRNN-SDR is not as stable in comparison to the other methods. This is reasonable because neural networks are generally more complex and require more data. As the training size increases, GRNN-SDR performs significantly better than the other methods and continues to achieve the best estimation accuracy for all training sizes studied in the experiment. Fig. 2(e)-(h) again demonstrate how the optimal prediction performance is achieved at the true structural dimension.

\begin{figure}
        \centering
        \subfloat[]{\includegraphics[width=0.4\linewidth]{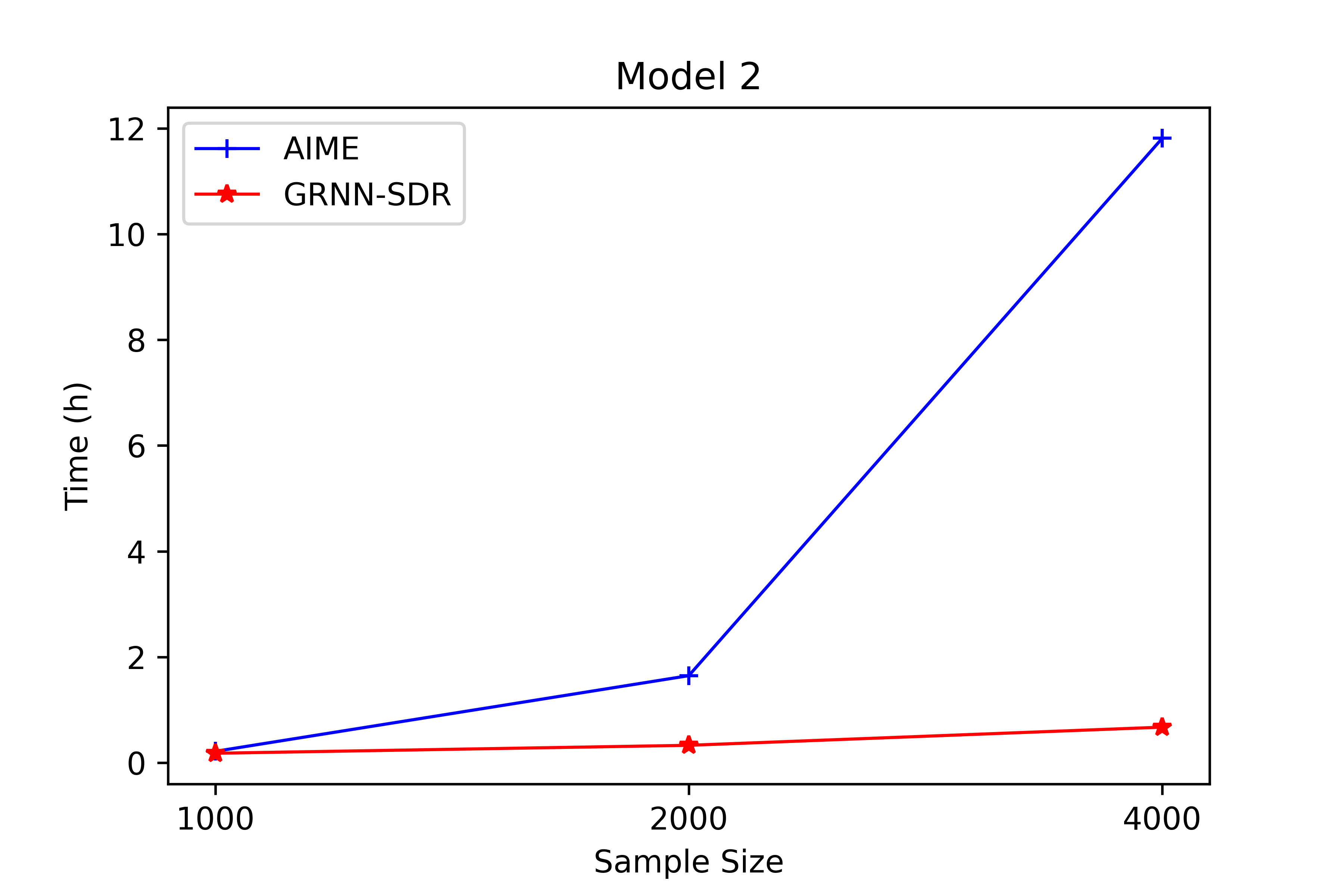}}%
        \hfill
        \subfloat[]{\includegraphics[width=0.4\linewidth]{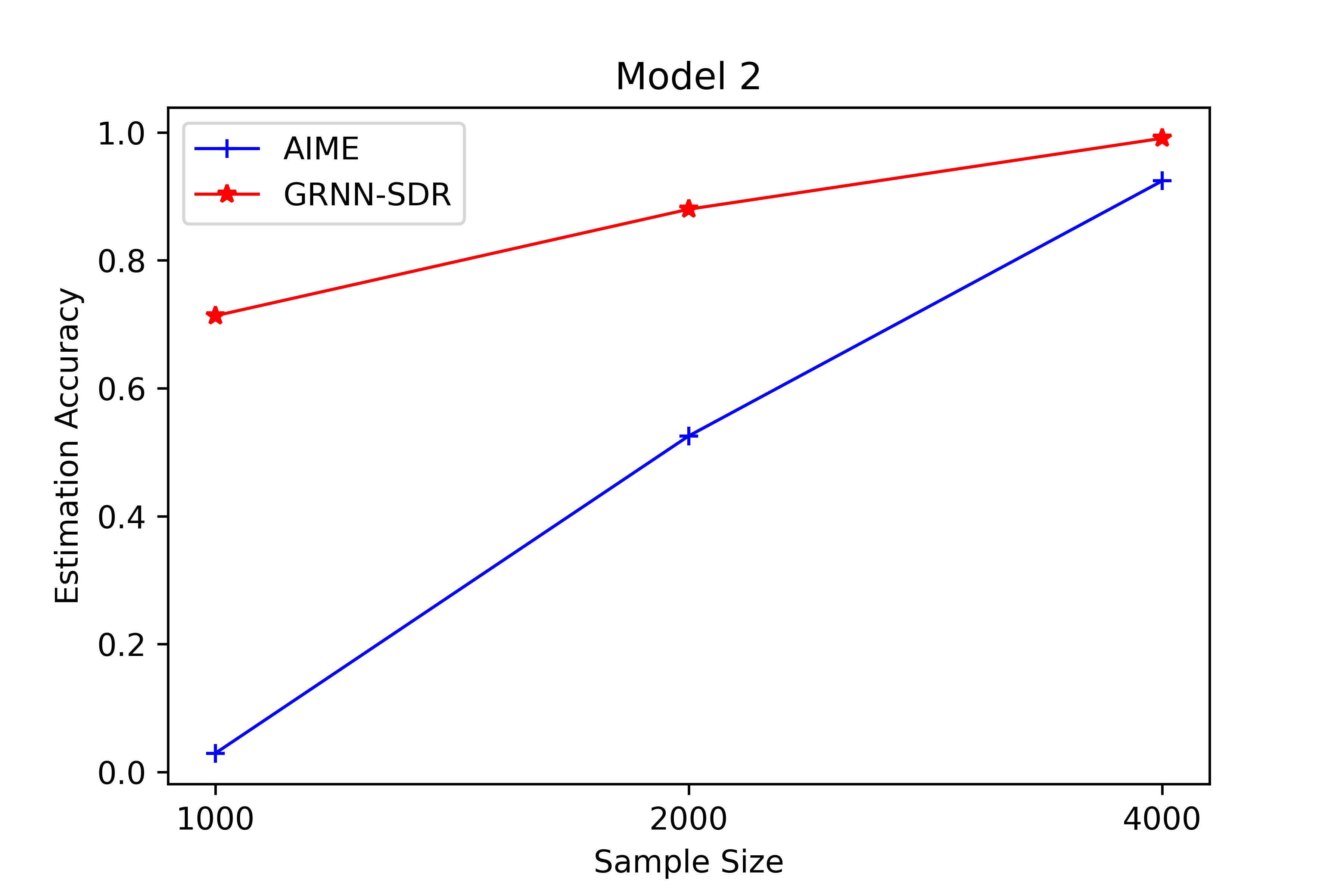}}
        \caption{Estimation accuracy ($r$) and computation time for Model 2}
          \label{figure3.2.1}
\end{figure}

In addition, we conducted a comparison of the computation time between GRNN-SDR and AIME using sample sizes of $1000$, $2000$, and $4000$. As illustrated in Fig. 3(a), the computation time of AIME increases dramatically as the sample size increases to $4000$, whereas the computation time of GRNN-SDR only increases linearly. This demonstrates the efficiency of GRNN-SDR, as its algorithmic complexity is $O(N)$ for a fixed neural network structure. Moreover, Fig. 3(b) also shows that, despite the difference in computation time, GRNN-SDR still achieves better estimation of the central space. 

\subsection{Simulation Study on Covariate Distribution}\label{ch3.4.4}
For a comparison study with state-of-the-art statistical methods in the literature, we replicate an experiment originated from \cite{xia2002adaptive} with Model 3. A recent experiment with more model comparisons is also shown in \cite{wang2020aggregate}. 
The data is randomly split into 1000 samples for training and validation ($80\%$ for training and $20\%$ for validation), and 1000 samples for testing.

\noindent \textbf{Model 3} \cite{xia2002adaptive}: Let $\beta_1 = (1,2,3,4,0,0,0,0,0,0)^{T}/\sqrt{30},\beta_2 = (-2,1,-4,3,1,2,0,0,0,0)^{T}/\sqrt{35}$, $\beta_3 = (0,0,0,0,2,-1,2,1,2,1)^{T}/\sqrt{15}$,
$\beta_4 = (0,0,0,$ $0,0,0,-1,-1,1,1)^{T}/2$. 
Let $A = (\beta_1, \beta_2, \beta_3, \beta_4)^{T}$, $Z = AX = (Z_1, Z_2, Z_3, Z_4)^{T}$, where $X \sim N_{p}(0,\Sigma)$, p = 10,  and $\Sigma = I_{p}$. 
\begin{equation*} 
Y = Z_1(Z_2)^{2} + Z_3Z_4 + 0.5\epsilon,
\end{equation*}
where $\epsilon$ is standard Gaussian noise independent of $X$. \\

\begin{figure} 
        \centering
        \subfloat[]{\includegraphics[width=0.4\linewidth]{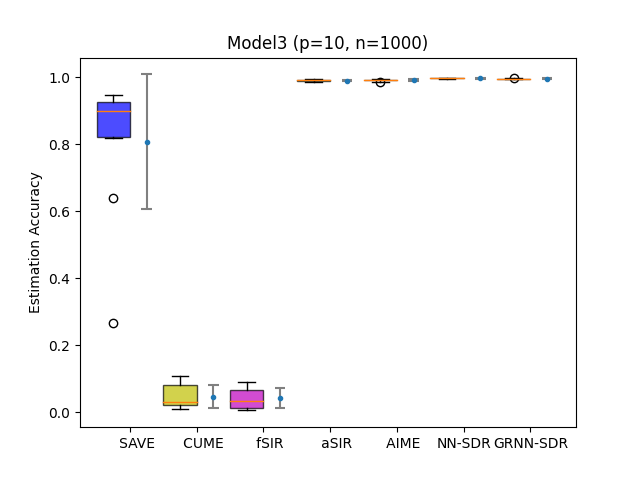}}%
        \hfill
        \subfloat[]{\includegraphics[width=0.4\linewidth]{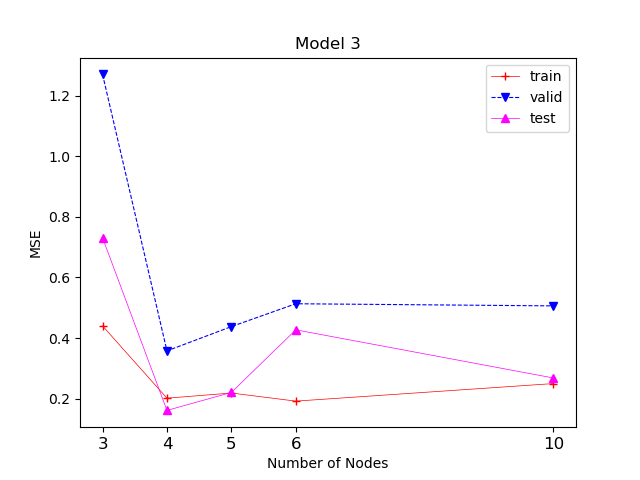}}
        \caption{Estimation accuracy ($r$) and MSE results for Model 3}
          \label{figure3.3}
\end{figure}

\subsection{Simulation Study on Feature Dimension Size}\label{ch3.4.5}
For Models 4 and 5, we aim to extend the analysis to higher feature dimensions, as many experimental studies in the state-of-the-art statistical methods literature have only reported results up to 20 features. To achieve this, we vary the feature dimensions and assess the performances of GRNN-SDR under different settings. The data is randomly split into 1000 samples for training and validation ($80\%$ for training and $20\%$ for validation), and 1000 samples for testing.

\noindent \textbf{Model 4}: Let $V = (V^{T}_{1}, V^{T}_{2}, \cdots, V^{T}_{p})^{T}$ be an orthogonal matrix where $V_1, V_2, \cdots, V_p$ are row vectors. Let $A = (V^{T}_{1}, V^{T}_{2}, V^{T}_{3})^{T}$, $Z = AX = (Z_1, Z_2,Z_3)^{T} $, where $X \sim N_{p}(0,\Sigma)$, and $\Sigma = I_{p}$. 
\begin{equation*} 
Y = 0.5Z_1Z_2 + \sin(Z_1-Z_3) + \cos(Z_2+Z_3) + 0.1\epsilon,
\end{equation*}
where $\epsilon$ is standard Gaussian noise independent of $X$.\\
\noindent \textbf{Model 5}: Same as Model 4 except $X \sim Uniform_{p}(-1,1)$.

\begin{figure*}[!ht]
  \small
  \centering
  \subfloat[]{
    \includegraphics[width=.22\textwidth,  trim=10 5 35 25,clip]{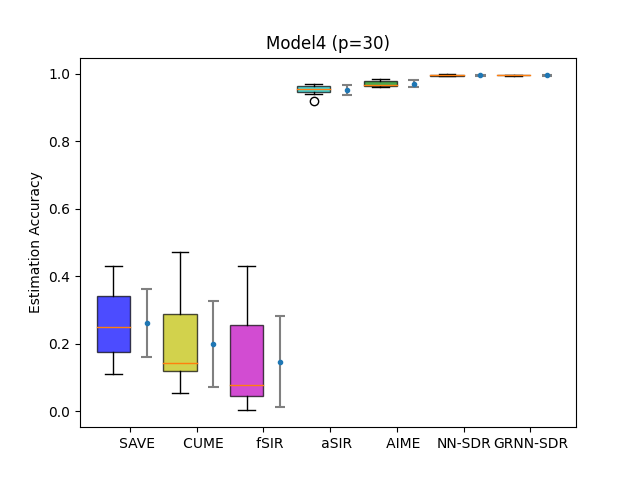}
    \label{fig:d}
  }
  \hfil
  \subfloat[]{
    \includegraphics[width=.22\textwidth, trim=10 5 35 25,clip]{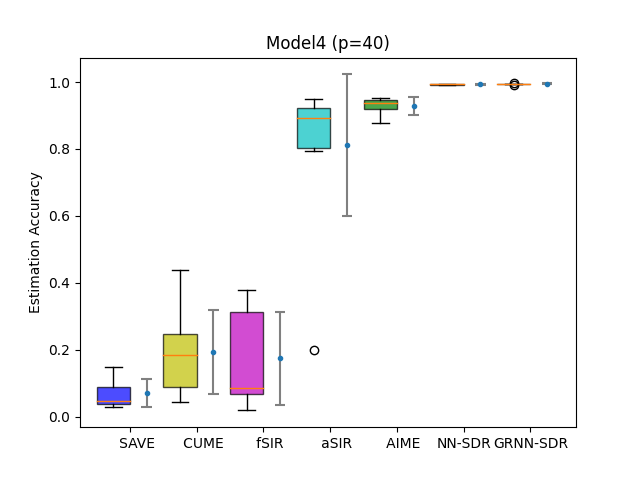}
    \label{fig:e}
  }
  \hfil
  \subfloat[]{
    \includegraphics[width=.22\textwidth, trim=10 5 35 25,clip]{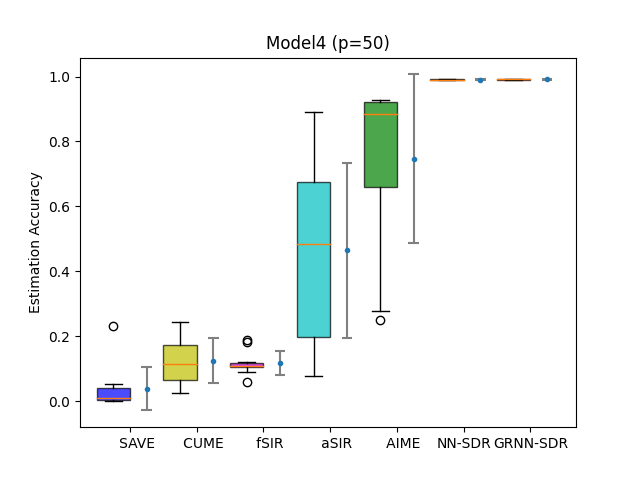}
    \label{fig:f}
  }
  \hfil
  \subfloat[]{
    \includegraphics[width=.22\textwidth,  trim=10 5 35 25,clip]{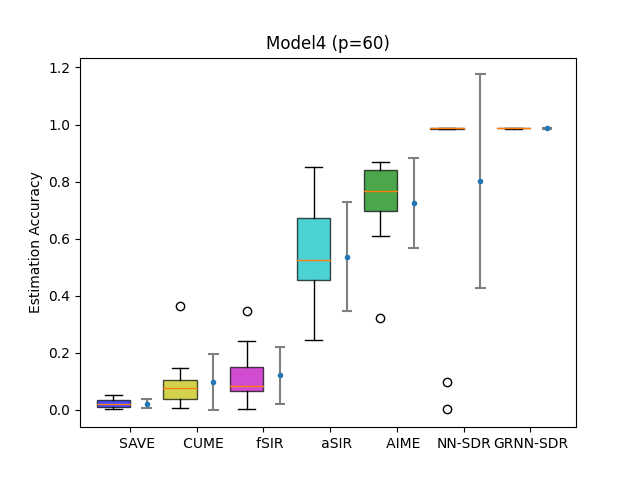}
    \label{fig:g}
  }
  \hfil
  \subfloat[]{
    \includegraphics[width=.22\textwidth,  trim=10 5 35 25,clip]{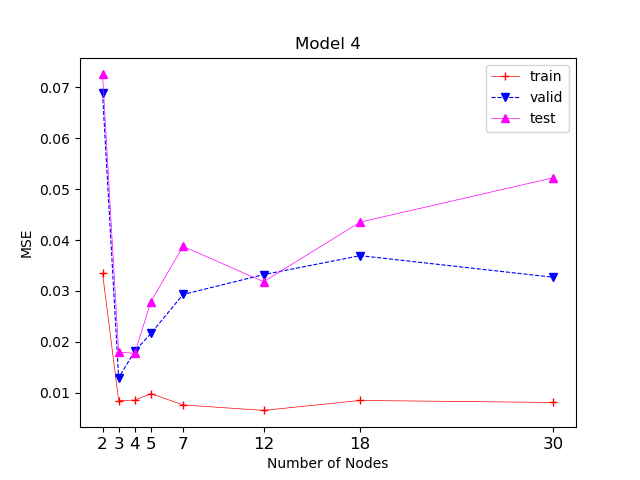}
    \label{fig:h}
  }
  \hfil
  \subfloat[]{
    \includegraphics[width=.22\textwidth,  trim=10 5 35 25,clip]{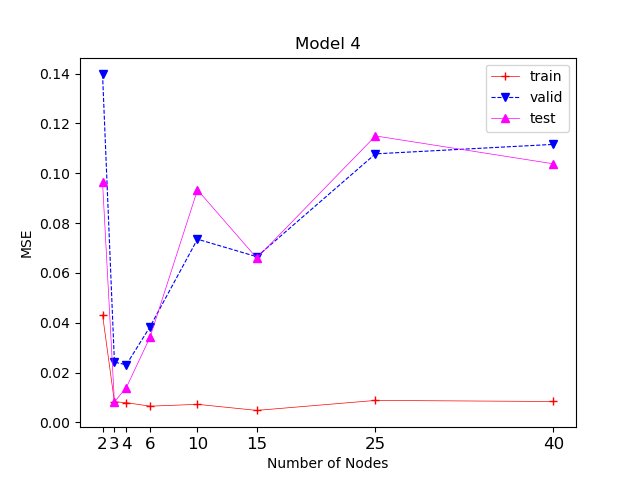}
    \label{fig:i}
  }
  \hfil
  \subfloat[]{
    \includegraphics[width=.22\textwidth,  trim=10 5 35 25,clip]{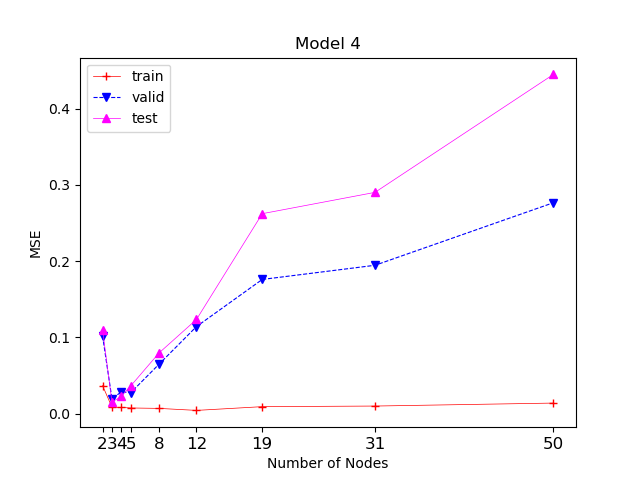}
    \label{fig:j}
  }
  \hfil
  \subfloat[]{
    \includegraphics[width=.22\textwidth,  trim=10 5 35 25,clip]{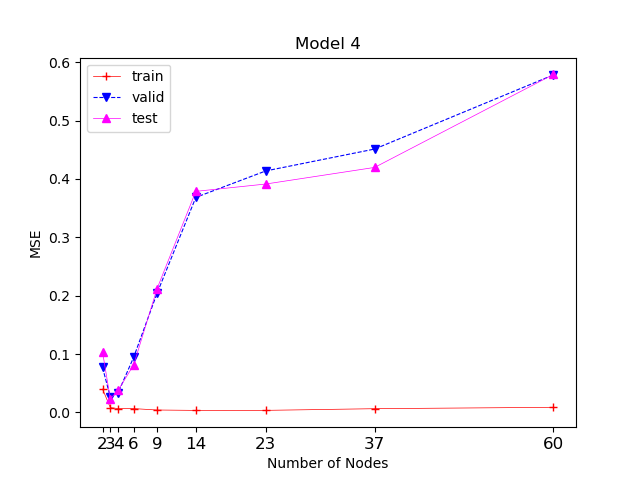}
    \label{fig:i}
  }
  \caption{Estimation accuracy ($r$) and MSE results for Model 4}
  \label{figure3.5}
\end{figure*}

The results presented in Fig. 5(a)-(d) demonstrate that GRNN-SDR and NN- SDR both consistently outperforms the other five methods in terms of estimation accuracy across all feature dimensions. However, under the uniform data distribution,  we can see from the top row of Fig.~\ref{figure3.6} that GRNN-SDR maintains substantially higher stability in comparison to NN-SDR especially as the feature dimension increases to 50 and 60. Fig. 5(e)-(h) and Fig. 6(e)-(h) shows how the number of nodes changes in each dynamic search to successfully identify the structural dimension. Moreover, the bottom row illustrates the algorithm's ability to dynamically search and obtain the structural dimension efficiently. For instance, when the feature dimension is 60, the structural dimension is reduced to 3 after only six iterations, illustrating the effectiveness and efficiency of the GRNN-SDR method.

\begin{figure*}[!ht]
  \small
  \centering
  \subfloat[]{
    \includegraphics[width=.22\textwidth,  trim=10 5 35 25,clip]{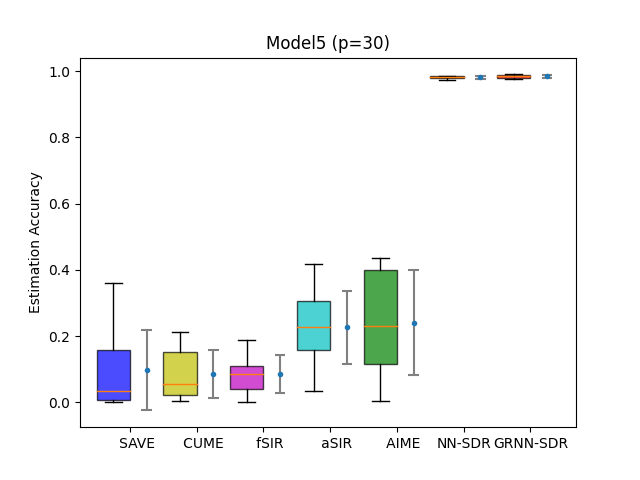}
    \label{fig:d}
  }
  \hfil
  \subfloat[]{
    \includegraphics[width=.22\textwidth, trim=10 5 35 25,clip]{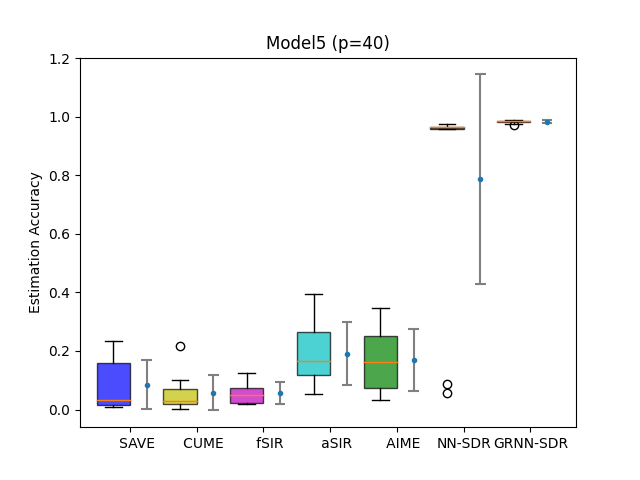}
    \label{fig:e}
  }
  \hfil
  \subfloat[]{
    \includegraphics[width=.22\textwidth, trim=10 5 35 25,clip]{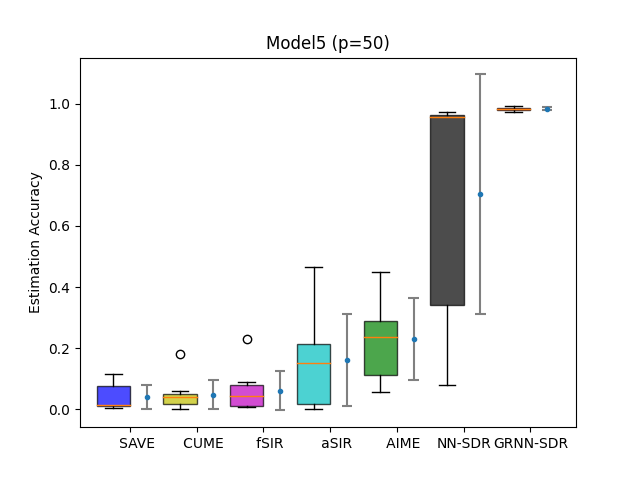}
    \label{fig:f}
  }
  \hfil
  \subfloat[]{
    \includegraphics[width=.22\textwidth,  trim=10 5 35 25,clip]{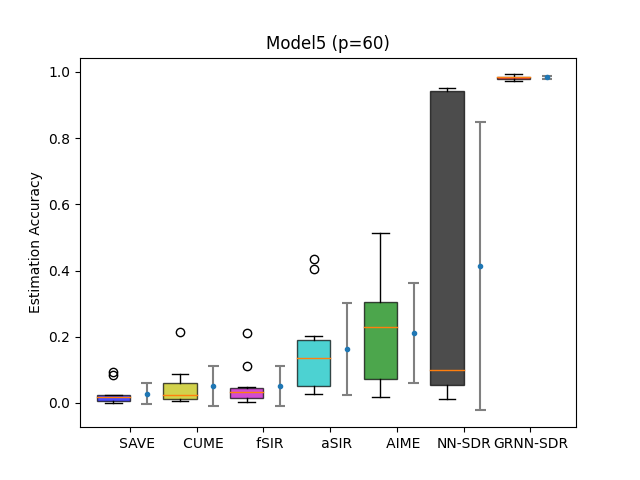}
    \label{fig:g}
  }
  \hfil
  \subfloat[]{
    \includegraphics[width=.22\textwidth,  trim=1 5 35 25,clip]{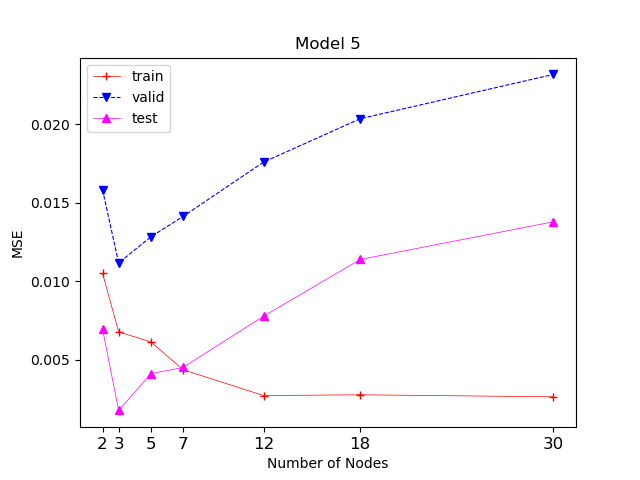}
    \label{fig:h}
  }
  \hfil
  \subfloat[]{
    \includegraphics[width=.22\textwidth,  trim=6 5 35 25,clip]{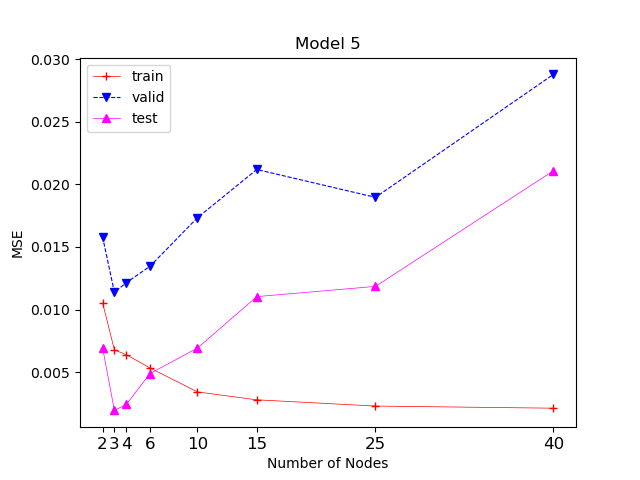}
    \label{fig:i}
  }
  \hfil
  \subfloat[]{
    \includegraphics[width=.22\textwidth,  trim=10 5 35 25,clip]{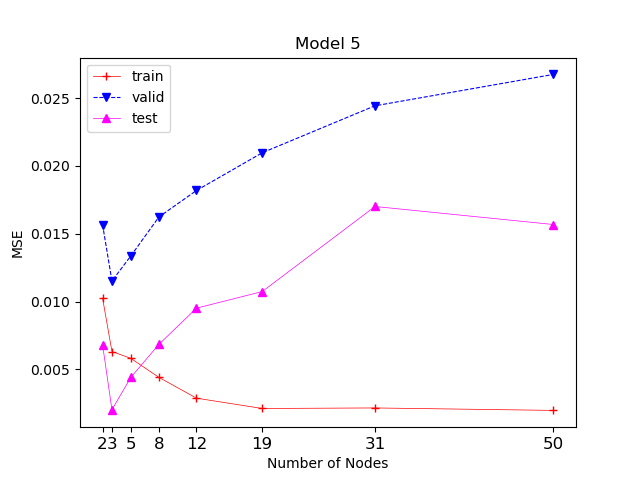}
    \label{fig:j}
  }
  \hfil
  \subfloat[]{
    \includegraphics[width=.22\textwidth,  trim=10 5 35 25,clip]{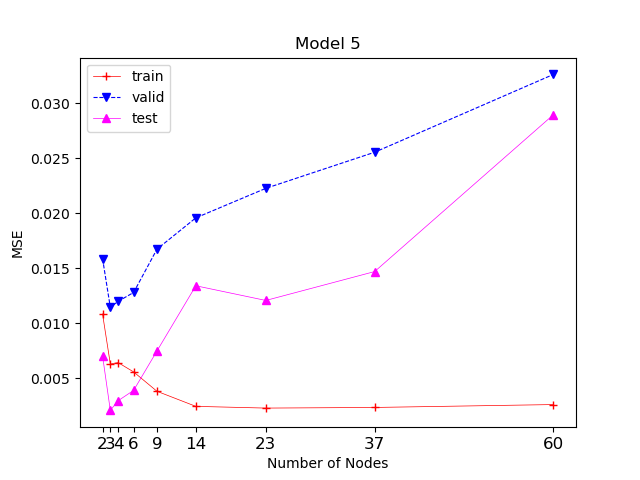}
    \label{fig:i}
  }
  \caption{Estimation accuracy ($r$) and MSE results for Model 5}
  \label{figure3.6}
\end{figure*}

\subsection{Simulation Study on Practical Dimension Reduction}\label{ch3.4.6}
Models 6 and 7 are used for simulation studies on approximation analysis for the GRNN-SDR method based on Model 1. In Model 6, we add another dimension to Model 1, but the term containing this term has minimal impact on $Y$. Therefore, under approximation, the structural dimension should be 5. In Model 7, we reduce the impact on a dimension in Model 1 indicating that, under approximation, the structural dimension should be 4. The data is split into a sample size of 2000 for training and validation ($80\%$ for training and $20\%$ for validation), and 1000 samples for testing. 

\noindent \textbf{Model 6}: Let $V = (V^{T}_{1}, V^{T}_{2}, \cdots, V^{T}_{p})^{T}$ be an orthogonal matrix where $V_1, V_2, \cdots, V_p$ are row vectors and $p = 20$. Let $A = (1.01V^{T}_{1}, 1.01V^{T}_{2}, 1.02V^{T}_{3}, 1.1V^{T}_{4}, 1.03V^{T}_{5}, 1.01V^{T}_{6})^{T}$, $Z = AX = (Z_1, Z_2, \cdots, Z_6)^{T} $, where $X \sim Uniform_{p}(-1,1)$.
\begin{align*} 
Y &= Z_3 - Z_1Z_5 + 0.5Z_2^{2} + (Z_3 + 0.5Z_4) / (1 + Z_1^{2})  + \exp(0.5(Z_3-Z_4)) \\ &\times \sin(Z_2 - Z_5 + 1.5Z_3) + 0.001(Z_6^{2})(Z_1^{2}+Z_2^{2}+Z_3^{2}) + 0.1\epsilon,
\end{align*}
where $\epsilon$ is standard Gaussian noise independent of $X$.

\noindent \textbf{Model 7}: Let $V = (V^{T}_{1}, V^{T}_{2}, \cdots, V^{T}_{p})^{T}$ be an orthogonal matrix where $V_1, V_2, \cdots, V_p$ are row vectors and $p = 20$. Let $A = (0.01V^{T}_{1}, 1.01V^{T}_{2}, 1.02V^{T}_{3}, 1.1V^{T}_{4}, 1.03V^{T}_{5})^{T}$, $Z = AX = (Z_1, Z_2, \cdots, Z_5)^{T} $, where $X \sim Uniform_{p}(-1,1)$.
\begin{align*} 
Y &= Z_3 - Z_1Z_5 + 0.5Z_2^{2} + (Z_3 + 0.5Z_4) / (1 + Z_1^{2})   + \exp(0.5(Z_3-Z_4)) \times \sin(Z_2 - Z_5 + 1.5Z_3) + 0.1\epsilon,
\end{align*}
where $\epsilon$ is standard Gaussian noise independent of $X$.

\begin{figure} 
        \centering
         \subfloat[]{\includegraphics[width=0.3\linewidth]{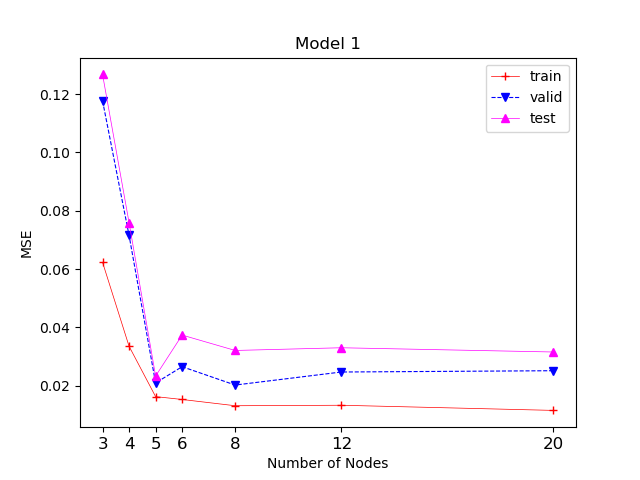}}%
        \hfill
         \subfloat[]{\includegraphics[width=0.3\linewidth]{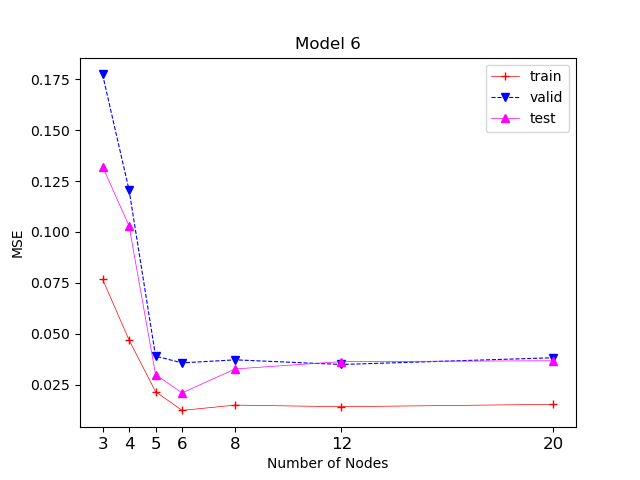}}
        \hfill
         \subfloat[]{\includegraphics[width=0.3\linewidth]{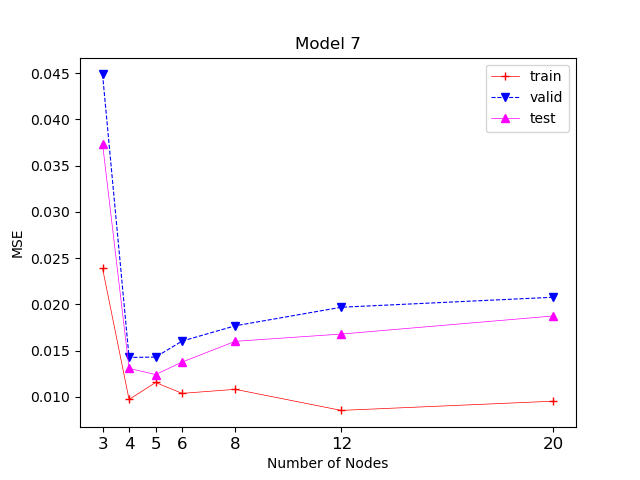}}
        \caption{MSE results comparison for noise level = 0.1}
          \label{figure3.8}
\end{figure}

Results from Fig.~\ref{figure3.8} illustrate the comparison of MSE on the training set, validation set, and test set for Model 1, Model 6, and Model 7 under a noise level of 0.1. As overfitting may occur during training and the test performance is unknown, we can analyze based on the validation set. Fig. 7(b) shows that $MSE_{va}$ is slightly higher when comparing five nodes and six nodes. However, the $MSE_{va}$ increases drastically from five nodes to four nodes. Thus, the difference observed between selecting five nodes and six nodes is the $\delta_N$-approximation, which reduced the number of nodes to five. This shows that even though there are six variables, GRNN-SDR can identify the five important variables from the sixth variable that is of little importance shown in the last term of $Y$ in Model 6. Similarly, Fig. 7(c) shows that $MSE_{va}$ is slightly higher when comparing four nodes and five nodes. However, the $MSE_{va}$ is drastically increased from four nodes to three nodes. The slight difference between selecting four nodes and five nodes is in the $\delta_N$-approximation, which reduced the number of nodes to four. So this shows that the GRNN-SDR method can detect and distinguish the first variable that is of less importance in comparison to the remaining four variables in Model 7 setting.

\section{Conclusion}\label{ch3.6}
We propose a new SDR method, GRNN-SDR, that offers several improvements compared to the existing methods in the literature. First, it uses a systematic incremental search to narrow down the structural dimension by simply adding another hidden layer to the neural network. For Barron classes, the GRNN-SDR method is effective for shallow neural networks as we show both theoretically and empirically. It can also be extended towards deep neural networks for more complex high-dimensional data applications such as imagery and text mining. Second, the algorithmic complexity of GRNN-SDR under a fixed neural network structure is $O(N)$, where $N$ is the sample size, while the best-performing method among the six other methods has an algorithmic complexity of $O(N^3)$. This reduces time and computation complexity significantly, as our experimental results have demonstrated. Third, for the estimation of the central space, GRNN-SDR substantially improves estimation accuracy compared to other methods, as demonstrated by extensive simulation studies using various models, including a widely studied model in the literature. 
In addition, GRNN-SDR identifies the structural dimension, which was presumed in most previous works.

The results of our extensive experiments under multiple different scenarios demonstrate the remarkable improvements provided by our proposed method. It is worth to note how the performance of neural networks can be limited by insufficient training sizes due to the nature of neural networks. For example, Fig.~\ref{figure3.2} shows that when the sample size is 500, the predicted structural dimension is 4 instead of the true structural dimension 5. However, GRNN-SDR successfully identifies the true structural dimension when the sample size increases to 2000. Furthermore, Fig.~\ref{figure3.2} also shows that when the sample size is 1000, the estimation accuracy for GRNN-SDR is $0.434$ and increases to $0.974$ when the sample size is increased to 4000. In contrast, even among the best of the other six methods, AIME has an estimation accuracy of only 0.03 with a sample size of 1000, increasing sharply to 0.761 with a sample size of 4000. In addition, the performances of GRNN-SDR are quite robust in comparison to other statistical methods. It seems fair to say that dimension reduction makes most sense in an approximation case. In this spirit, we extend the proposed method to analyze the $\delta_N$-approximation case, which may be more applicable in practice. Results from Section~\ref{ch3.4.6} show that GRNN-SDR is capable of obtaining a structural dimension under the $\delta_N$-approximation. 

To further validate our method, future research may aim at empirical studies on visual sequence classification, face recognition, causal inference, etc. In addition, more work is needed for the understanding of the practical dimension reduction framework. For instance, what is a proper choice of $\delta_N$ and how should one estimate $\delta_N$ based on the data. Results on these questions may help bridge the gap between the theoretical strengths and real applications.

\section*{Acknowledgments}
This paper is in celebration of Professor Andrew Barron's retirement in 2024, who has had tremendous impacts and inspirations to the authors over years.

\section{Appendix}
\subsection{Proof of Proposition 1}\label{proof: prop1}

\begin{proof}
According to Theorem 1 in \cite{barron1993universal}, if the support set of g(z) $\subseteq \mathcal{B}$, where $\mathcal{B}$ is a closed bounded set, then there exists a neural network of one hidden layer with $m$ nodes, $g_m(z,\bar{\theta}) = \sum_{i=1}^{m}\bar{\tau}_{i}\phi(\bar{u}_{i}^{T}z + \bar{v}_{i}) + \bar{\tau}_0$, where $\phi(\cdot)$ is any sigmoidal function (e.g. tanh) such that $\big\|g(\cdot) - g_m(\cdot,\bar{\theta}) \big\| \leq  \frac{2C_g}{\sqrt{m}}$.

Define a neural network of two hidden layers, $f_{d,m}(x,\bar{\theta},\beta_d)$, where the activation function in the first hidden layer is an identity function, $I(\cdot)$, and 
\begin{align} 
f_{d,m}(x,\bar{\theta},\beta_d) \overset{\Delta}{=}  g_{m}(\beta_d^{T}x, \bar{\theta}) =\sum_{i=1}^{m}\tau_{i}\phi(u_{i}^{T}I(\beta_d^{T}x)+v_{i}) + \tau_0.
\end{align}
Then,
\begin{align}
\big\|f(\cdot)-f_{d,m}(\cdot,\bar{\theta},\beta_d)\big\| &= \big\|g(\beta_d^{T}\cdot) - g_{m}(\beta_d^{T}\cdot, \bar{\theta})\big\| \nonumber \\
&= \big\|g(\cdot) - g_{m}(\cdot, \bar{\theta})\big\|  \nonumber \\
&\leq \frac{2C_g}{\sqrt{m}} .
\end{align}
\end{proof}

\subsection{Proof of Proposition 2}\label{proof: prop2}
\begin{proof}
Let 
\begin{align}
    z &= \beta_d^Tx \in [-M,M]^d, \\
    \tilde{z} &= \frac{1}{2M}(z + M\cdot\vec{1}) \in [0,1]^d, \\
    g(z) &= g(2M\tilde{z} - M\cdot\vec{1}) \\
    &\overset{\Delta}{=} \tilde{g}(\tilde{z}) \in W^{s,\infty}([0,1]^d).
\end{align}
Based on Theorem 5.1 in \cite{de2021approximation}, there exists a tanh neural network, $\hat{\tilde{g}}^N$ of two hidden layers such that 
\begin{align}
    \|\tilde{g} - \hat{\tilde{g}}^N\|_{L^{\infty}([0,1]^d)} \leq (1+\delta)\frac{\mathcal{C}_{d,s,\tilde{g}}}{N^s}
\end{align}
so 
\begin{align}
     \|f - \hat{f}^N\|_{L^{\infty}(B)} &=  \|g(\beta_d^T\cdot) - \hat{g}^N(\beta_d^T\cdot)\|_{L^{\infty}(B)} \nonumber \\
     &= \|g(\cdot) - \hat{g}^N(\cdot)\|_{L^{\infty}([-M,M]^d)} \nonumber \\
     &=  \|\tilde{g} - \hat{\tilde{g}}^N\|_{L^{\infty}([0,1]^d)}  \nonumber \\
     &\leq (1+\delta)\frac{\mathcal{C}_{d,s,\tilde{g}}}{N^s} \nonumber \\
     &= (1+\delta)\frac{(2M)^s\mathcal{C}_{d,s,g}}{N^s},
\end{align}
where $\mathcal{C}_{d,s,g} = \frac{1}{s!}(\frac{3d}{2})^s|g|_{W^{s,\infty}([-M,M]^d)}$.
\end{proof}

\subsection{Proof of Theorem 1}\label{proof: thm1}
The proof for Theorem 1 relies on Lemma~\ref{lem3.1} and Lemma~\ref{lem3.2}. 

\begin{lemma}\label{lem3.1}
If (\ref{eq.5}) $\sim$ (\ref{eq.9}) holds, then for each $(\theta, w_d)$ in the continuous parameter set $\Theta_{d,m,\delta,C}$, there is a $(\theta^*, w_d^*)$ in the discrete set $\Theta_{d,m,\epsilon,\delta,C}$ such that uniformly for $x \in [-1,1]^p$,
\begin{equation}
        |f_{d,m}(x,\theta,w_d) - f_{d,m}(x,\theta^*,w^*_d)| \leq aC\epsilon,
\end{equation}
and hence
\begin{equation}
        \|f_{d,m}(\cdot,\theta,w_d) - f_{d,m}(\cdot,\theta^*,w^*_d)\| \leq aC\epsilon ,
\end{equation}
where $f_{d,m}(x,\theta,w_d)$ is the family of sigmoidal networks of the form given in (\ref{eq.2}).
\end{lemma} 

\begin{proof}
Let $(\theta, w)$ and $(\theta^*, w^*)$ be parameter vectors in $\Theta_{d,m, \delta, C}$ and $\Theta_{d,m,\epsilon, \delta, C}$ for which (\ref{eq.5}) $\sim$ (\ref{eq.9}) holds, respectively. Consider the difference between the values of the corresponding network functions, 
\begin{align}
    &f_{d,m}(x,\theta,w_d) - f_{d,m}(x,\theta^*,w^*_d) = \nonumber \\
    &\sum^m_{i=1}(\tau_i\phi(u^T_iw^T_dx + v_i) - \tau^*_i\phi(u^{*T}_iw^{*T}_dx + v^*_i)) + (\tau_0 - \tau^*_0).
\end{align}
Then,
\begin{align}
        |f_{d,m}(x,\theta, w_d) - f_{d,m}(x,\theta^*, w^*_d)| 
        &\leq \sum^m_{i=1}|\tau_i|\cdot|\phi(z_i) - \phi(z^*_i)|  + \sum^m_{i=1}|\tau_i - \tau^*_i|\cdot|\phi(z^*_i)| + |\tau_0 - \tau^*_0| \nonumber  \\
        &\leq a_1\sum^m_{i=1}|\tau_i|\cdot|z_i - z^*_i| + a_0\sum^m_{i=1} |\tau_i - \tau^*_i| + |\tau_0 - \tau^*_0|,
\end{align}
where $z_i = u^T_iw^T_dx + v_i$, $z^*_i = u^{*T}_iw^{*T}_dx + v^*_i$. 

\noindent Since
\begin{align}
|z_i - z^*_i| &\leq |(u^T_iw^T_d - u^{*T}_iw^{*T}_d)x| + |v_i - v^*_i|  \nonumber \\
&\leq \|u^T_iw^T_d - u^{*T}_iw^{*T}_d\|_1 + \epsilon  \hspace{0.6cm} \nonumber \\ 
&\leq \|(u^T_i - u^{*T}_i)w^{T}_d\|_1 + \|u^{*}_i(w^T_d - w^{*T}_d)\|_1 + \epsilon  \nonumber \\
&\leq\delta \|u_i - u^*_i\|_1 + \delta\sum^d_{j=1}\|w^{(j)} - w^{(j)*}\|_1 + \epsilon \nonumber  \\
&\leq\delta\epsilon + \delta d \epsilon + \epsilon,
\end{align}
we have
\begin{align}
    |f_{d,m}(x,\theta, w_d) - f_{d,m}(x,\theta^*, w^*_d)| &\leq a_1(\delta(d+1)+1)\epsilon C + a_0C\epsilon + C\epsilon \nonumber  \\
    &= (a_1(\delta(d+1)+1) +a_0 + 1)C\epsilon  \nonumber  \\
    &= aC\epsilon.
\end{align}
\end{proof}

\begin{lemma}\label{lem3.2}
For each $\delta > 0$, $\epsilon >0$, and $C \geq 1$, there is a set $\Theta_{d,m,\epsilon,\delta,C}$ that satisfies (\ref{eq.5}) $\sim$ (\ref{eq.9}) and has cardinality bounded by

\begin{align}
|\Theta_{d,m,\epsilon, \delta, C}| &\leq (\frac{2e(\delta + \epsilon)}{\epsilon})^{dp} \cdot (\frac{2e(\delta + \epsilon)}{\epsilon})^{md} \cdot 
(\frac{2(\delta + \epsilon)}{\epsilon})^{m} \cdot (\frac{2e(1 + \epsilon)}{\epsilon})^{m}\cdot(\frac{v + 2 + \epsilon}{\epsilon}),
\end{align}
where $v = \max\{a_0, a_1\}$ and $e$ is the Euler's number.
\end{lemma} 

\begin{proof}
The proof is similar to \cite{barron1994approximation}. First, consider a rectangular grid for the coordinates of $w^{(j)} \in \mathbb{R}^p$, where $j = 1, \cdots, d$ spaced at width $\epsilon / p$, spaced at width $\epsilon / d$ for the coordinates of $u_i$, spaced at width $\epsilon$ for $v_i$ spaced at width $C\epsilon$ for $\tau_0$, and spaced at width $C\epsilon/m$ for $\tau_i$, for $i = 1, \cdots, m$. Intersecting the grid with $\Theta_{d,m, \delta, C}$ allows the set $\Theta_{d,m,\epsilon, \delta, C}$ to satisfy the requirements in (\ref{eq.5}) $\sim$ (\ref{eq.9}). Since $\Theta_{d,m, \delta, C}$ is a cartesian product of constraint sets for the $u$'s and $v$'s and $\tau$'s, so the desired cardinality is a product of the corresponding counts. 

Then, bound the number of grid points in the simplex $S_\delta = \{u \in \mathbb{R}^d: \|u\|_1 \leq \delta \}$, where the grid points are spaced at width $\epsilon/d$ in each coordinate. The union of the small cubes that intersect $S_\delta$ is contained in $S_{\delta + \epsilon}$. (Any point $u$ in this union has $l_1$ distance less than $\epsilon$ from a point $u'$ in $S_{\delta}$ for $\|u\|_1 \leq \|u'\|_1 + \|u - u'\|_1 \leq \delta + \epsilon$, so $u$ is in $S_{\delta + \epsilon}$.) The volume of this union of cubes is the number of cubes times the volume of a cube, i.e., $(\epsilon/d)^d$. Also, the volume of the covering simplex is $(2(\delta+\epsilon))^d/d!$. So the number of cubes that intersect $S_\delta$ is no greater than 
\begin{align}
\frac{(2(\delta+\epsilon))^d/d!}{(\epsilon/d)^d} = (2d(\delta+\epsilon)/\epsilon)^d/d! 
   \leq (2e(\delta + \epsilon)/\epsilon)^d.
\end{align}

For $m$ such parameter vectors $u_i$, the total count is bounded by $(2e(\delta + \epsilon)/\epsilon)^{md}$. Similarly, the total count for the $v_i$'s is bounded by $(2(\delta + \epsilon)/\epsilon)^{m}$, the total count for $\tau_i$'s is bounded by $(2e(1+\epsilon)/\epsilon)^m$, and the total count for $\tau_0$ is bounded by $(v + 2C + C\epsilon)/(C\epsilon) \leq (v + 2 + \epsilon)/\epsilon$ for $C \geq 1$. Let $U_\delta = \{w^{(j)} \in \mathbb{R}^p: \|w^{(j)}\|_1 \leq \delta\}$. Then, the number of cubes that intersect $U_\delta$ is no greater than $(2p(\delta + \epsilon)/\epsilon)^p/p! \leq (2e(\delta + \epsilon)/\epsilon)^p$. For $d$ such parameter vectors $w^{(j)}$, the total count is bounded by $(2e(\delta + \epsilon)/\epsilon)^{dp}$. 

Finally, taking the product of the counts for $w^{(j)}, u_i, v_i, \tau_i$, and $\tau_0$ gives the cardinality
\begin{align}
    |\Theta_{d,m,\epsilon, \delta, C}| \leq (\frac{2e(\delta + \epsilon)}{\epsilon})^{dp} \cdot (\frac{2e(\delta + \epsilon)}{\epsilon})^{md} \cdot
    (\frac{2(\delta + \epsilon)}{\epsilon})^{m} \cdot (\frac{2e(1 + \epsilon)}{\epsilon})^{m}\cdot(\frac{v + 2 + \epsilon}{\epsilon}). 
\end{align}

\end{proof}

\noindent Now we prove Theorem 1. 
\begin{proof}
Choose large enough $\delta_0 >1$. For $\delta \geq \delta_0$, $0 < \epsilon \leq 1$, we have 
\begin{equation}
    \log |\Theta_{d,m,\epsilon, \delta, C}| \leq Q\log\frac{2e(1 + \delta)}{\epsilon},
\end{equation}
where
\begin{align}
Q &\overset{\Delta}{=} dp + md + 2m + 1  \\
&= (p+m)d + 2m + 1.
\end{align}
Define $R_{d,m,N}(f) \overset{\Delta}{=} \min_{(\theta,w) \in \Theta_{d,m,\epsilon, \delta, C}}\|f(\cdot) - f_{d,m}(\cdot, \theta, w)\|^2 + \frac{\lambda}{N}\log |\Theta_{d,m,\epsilon, \delta, C}|$, where $\lambda$ where $\lambda$ is a constant greater than 1. Then,
\begin{align}
R_{d,m,N}(f) &\leq \|f - f_{d,m}\|^2 + \frac{\lambda}{N}\log |\Theta_{d,m,\epsilon, \delta, C}| \nonumber \\
&\leq 2(\|f(\cdot) - f_{d,m}(\cdot, \bar{\theta}, \bar{w})\|^2  + \|f_{d,m}(\cdot,\bar{\theta}, \bar{w}) 
- f_{d,m}(\cdot,\theta^*,\hat{w}^*_{d})\|^2)  + \frac{\lambda}{N}\log|\Theta_{d,m,\epsilon,\delta,C}| \nonumber \\
&\leq 2(\frac{{4C_g}^2}{m}) + 2(aC \epsilon)^2  + \frac{\lambda}{N}Q\log\frac{2e(1+ \delta)}{\epsilon}.
\label{eq31}
\end{align}

\noindent Note that in (\ref{eq31}), the first term comes from proof for Proposition 1, the second term comes from Lemma~\ref{lem3.1}, and the third term comes from Lemma~\ref{lem3.2}. 

\noindent By taking the partial derivative with respect to $\epsilon$, we have 
\begin{align}
    4a^2C^2\epsilon - \frac{\lambda Q}{N}\frac{1}{\epsilon} &= 0 \\
 \Rightarrow \epsilon &= \frac{1}{2aC}(\frac{\lambda Q}{N})^{1/2} \sim \left(\frac{(m+p)d}{N}\right)^{1/2},
\end{align}
which optimizes the bound. Therefore,
\begin{align}
R_{d,m,N}(f) &\leq \frac{2(4{C_g}^2)}{m} + 2aC\frac{1}{4a^2C^2}\frac{\lambda Q}{N} +\frac{\lambda Q}{N}\log\left(\frac{2e(\delta + 1)}{2aC}\sqrt{\frac{N}{\lambda Q}}\right) \nonumber \\ 
&= O\left(\frac{C_g^2}{m}\right) + O\left(\frac{Q}{N}\log \frac{N}{Q}\right) \nonumber \\
&= O\left(\frac{C_g^2}{m}\right) + O\left(\frac{(m+p)d}{N}\log(\frac{N}{(m+p)d})\right).
\end{align}

Thus, based on the theorem convergence for complexity regularization in \cite{barron1994approximation},
\begin{align}
E\big\|f(\cdot) - f_{d,m}(\cdot,\hat{\theta}_{d,m,N},\hat{w}_{d,m,N})\big\|^2 &\leq O\left(R_{d,m,N}(f)\right)  \nonumber \\
&\leq O\left(\frac{C_g^2}{m}\right) + O\left(\frac{(m+p)d}{N}\log(\frac{N}{(m+p)d})\right).
\end{align}
\end{proof}

\subsection{Proof of Corollary 1}\label{proof: cor1}
\begin{proof}
Let $\tilde{g}(w^Tx) = f_{k,m}(x,\hat{\theta}_{k,N},\hat{w}_{k,N})$, where $ w = \hat{w}_{k,N}$. Based on \textbf{Assumption 2}, 
\begin{equation}
    \|f(\cdot) - \tilde{g}(w^T\cdot)\|^2 \geq c^2.
\end{equation}
Thus, $E\big\|f(\cdot) - f_{k,m}(\cdot,\hat{\theta}_{k,N},\hat{w}_{k,N})\big\|^2  \geq c^2$ holds when $k<d$. 

\noindent The proof for the case when $k \geq d$ is similar as in Theorem 1.
\end{proof}

\subsection{Proof of Theorem 2}\label{proof: thm2}
\begin{proof}
Note that it is suffice to show:
\begin{enumerate}
    \item If $k<d$, $P(CR(k)>CR(d)) \rightarrow 1$
    \item If $k>d$, $P(CR(k)>CR(d)) \rightarrow 1$
\end{enumerate}
Observe that, 
\begin{align}
    CR(k) - CR(d) &= \frac{1}{n_{va}}\sum^{n_{va}}_{i=1}(Y_i - \hat{f}_k(X_i))^2 + k*pen(N, n_{va})  -\frac{1}{n_{va}}\sum^{n_{va}}_{i=1}(Y_i - \hat{f}_d(X_i))^2 + d*pen(N, n_{va}) \nonumber \\
    &=\frac{1}{n_{va}}(\sum^{n_{va}}_{i=1}(f(X_i) - \hat{f}_k(X_i))^2 -\sum^{n_{va}}_{i=1}(f(X_i) - \hat{f}_d(X_i))^2 ) \nonumber \\
    &+ \frac{2}{n_{va}}\sum^{n_{va}}_{i=1}\epsilon_i(\hat{f}_k(X_i)-\hat{f}_d(X_i)) + (k-d)*pen(N, n_{va}) \nonumber \\ 
    &\overset{\Delta}{=} I + II + (k-d)*pen(N, n_{va}), \label{eq3.58}
\end{align}
where 
\begin{align*}
    I &\overset{\Delta}{=}\frac{1}{n_{va}}(\sum^{n_{va}}_{i=1}(f(X_i) - \hat{f}_k(X_i))^2 -\sum^{n_{va}}_{i=1}(f(X_i) - \hat{f}_d(X_i))^2 ), \\
    II &\overset{\Delta}{=} \frac{2}{n_{va}}\sum^{n_{va}}_{i=1}\epsilon_i(\hat{f}_k(X_i)-\hat{f}_d(X_i)).
\end{align*} 

Based on the results in Corollary 1, we know when $k<d$, $I$ is bounded away from 0 in probability and $II$ is of order $O_p(\frac{1}{\sqrt{n_{va}}})$. Then, because of the choice of the penalty such that $pen(N, n_{va}) \rightarrow 0$ as $N \rightarrow \infty$ and $n_{va} \rightarrow \infty$, we conclude that $P(CR(k) > CR(d)) \rightarrow 1$. 

When $k>d$, $I$ is of order $O_p(\frac{1}{N}\log{N})^{\frac{1}{2}}$ and $II$ is again of order $O_p(\frac{1}{\sqrt{n_{va}}})$. Since $(k-d)*pen(N, n_{va}) = (k-d)((\frac{1}{N}\log{N})^{\frac{1}{2}}+\frac{1}{\sqrt{n_{va}}})a_{N,n_{va}}$ with $a_{N,n_{va}} \rightarrow \infty$, we know $I + II$ is dominated by the penalty term. Thus, $P(CR(k) > CR(d)) \rightarrow 1$.

This completes the proof of the theorem.
\end{proof}

\subsection{Proof for Proposition 3}\label{proof: prop3}
\begin{proof}
\begin{enumerate}
    \item According to Theorem 1 in \cite{barron1993universal}, if the support set of g(z) $\subseteq \mathcal{B}$, then there exists a neural network of one hidden layer with $m$ nodes, $g_m(z) = g_m(z,\bar{\theta}) = \sum_{i=1}^{m}\bar{\tau}_{i}\phi(\bar{u}_{i}^{T}z + \bar{v}_{i}) + \bar{\tau}_0$, where $\phi(\cdot)$ is any sigmoidal function (e.g. tanh) such that $\big\|g(\cdot) - g_m(\cdot,\bar{\theta}) \big\| \leq  \frac{2C_g}{\sqrt{m}}$.

Define a neural network of two hidden layers, $f_{d,m}(x,\bar{\theta},\beta_d)$, where the activation function of the first hidden layer is an identity function, $I(\cdot)$, and
\begin{align} 
f_{d,m}(x,\bar{\theta},\beta_d) \overset{\Delta}{=}  g_{m}(\beta_d^{T}x, \bar{\theta}) =\sum_{i=1}^{m}\tau_{i}\phi(u_{i}^{T}I(\beta_d^{T}x)+v_{i}) + \tau_0.
\end{align}
Then,
\begin{align}
\big\|f(\cdot)-f_{d,m}(\cdot,\bar{\theta},\beta_d)\big\| &\leq \big\|f(\cdot) - g(\beta_d^T\cdot)\big\| + \big\|g(\beta_d^T\cdot) - g_{m}(\beta_d^{T}\cdot, \bar{\theta})\big\| \nonumber \\
&< \delta_N + \frac{2C_g}{\sqrt{m}} .
\end{align}

\item For any $m, \theta$, and $\alpha_k$ where $k \leq d-1$, if $k <d-1$, define $\alpha = (\alpha_k, \textbf{0}) \in \mathbb{R}^{p \times (d-1)}$, where $\textbf{0} \in \mathbb{R}^{p \times (d-1-k)}$ is a zero matrix, else $\alpha = \alpha_k$. Let 
\begin{align}
    \tilde{g}(\alpha^Tx) =  f_{k,m}(x,\theta,\alpha_k),
\end{align}
then according to \textbf{Assumption 4}, 
\begin{align}
    \big\|f(\cdot)-f_{k,m}(\cdot,\theta,\alpha_k)\big\| = \big\|f(\cdot)- \tilde{g}(\alpha^T\cdot)\big\| 
    \geq c_N.
\end{align}
\end{enumerate}
\end{proof}

\bibliography{SDRarXiv}
\end{document}